\documentclass[12pt]{article} 
\usepackage{amsthm}
\usepackage{newtxtext,newtxmath, amsmath, amsfonts, bm,mathtools}
\usepackage{paralist, url, natbib, parskip}
\usepackage{adjustbox, todonotes} % , 
\usepackage{graphicx, subfig}
\usepackage{rotating, booktabs, multirow}
\usepackage[makeroom]{cancel}
\usepackage[section]{placeins}
\usepackage[linesnumbered,lined,boxed,commentsnumbered]{algorithm2e}
\usepackage{xcolor}
\usepackage{placeins}
\usepackage{arydshln}
\usepackage{hyperref}
\usepackage{pgf,tikz,pgfplots}
\pgfplotsset{compat=1.15}
\usepackage{mathrsfs}
\usepackage{arydshln}
\usetikzlibrary{arrows}
\usetikzlibrary{decorations.markings}
\usetikzlibrary{positioning, shadows}
\graphicspath{{figures/}}

\usepackage[margin=3cm]{geometry}
\usepackage[margin=2cm]{caption}
\usepackage{titling}
\usepackage{setspace}
\usepackage{tikz}
\usetikzlibrary{trees}

\DeclareCaptionStyle{italic}[justification=centering]{labelfont={bf},textfont={it},labelsep=colon}
\SetKwInOut{Parameter}{parameter}
\newcommand{\Hom}{\text{Hom}}

\usepackage{mathtools}
\mathtoolsset{showonlyrefs}
\usepackage{todonotes}
\usepackage{array}
\usepackage{thmtools,thm-restate}

\newcolumntype{R}[1]{>{\raggedleft\let\newline\\\arraybackslash\hspace{0pt}}m{#1}}
\newcolumntype{L}[1]{>{\raggedright\let\newline\\\arraybackslash\hspace{0pt}}m{#1}}
\newcolumntype{C}[1]{>{\centering\let\newline\\\arraybackslash\hspace{0pt}}m{#1}}

\theoremstyle{definition}

\renewcommand{\deg}{\text{deg }}
\newcommand{\density}{\text{density}}
\newcommand{\ssq}{\text{ssq}}
\newcommand{\sq}{\text{sq}}

\newtheorem{theorem}{Theorem}[section]
\newtheorem{lemma}[theorem]{Lemma}
\newtheorem{corollary}[theorem]{Corollary}
\newtheorem{proposition}[theorem]{Proposition}
\newtheorem{definition}[theorem]{Definition}
\newtheorem{remark}[theorem]{Remark}

\newcommand{\cheng}[1]{}
\newcommand{\SK}[1]{}
\def\vertex{\tikz[baseline=.1ex]{
\fill (0,0.15) circle (2pt) coordinate (A);
}}

\def\edge{\tikz[baseline=.1ex]{
\fill (0,0.15) circle (2pt) coordinate (A);
\fill (1.5ex,0.15) circle (2pt) coordinate (B);
\draw (A)--(B);}
}

\def\twoedges{\tikz[baseline=.1ex]{
\fill (0,0.1) circle (2pt) coordinate (A);
\fill (1ex,0.25) circle (2pt) coordinate (B);
\fill (2ex,0.1) circle (2pt) coordinate (C);
\draw (A)--(B);
\draw (B)--(C)}
}

\usetikzlibrary{calc}
\title{Graphons of Line Graphs}
\author{Sevvandi Kandanaarachchi, Cheng Soon Ong}
\date{}

\begin{document}

\maketitle

\begin{abstract}
We consider the problem of estimating graph limits, known as graphons, from observations of sequences of sparse finite graphs.
In this paper we show a simple method that can shed light on a subset of sparse graphs. The method involves mapping the original graphs to their \textit{line graphs}. 
We show that graphs satisfying a particular property are sparse, but give rise to dense line graphs. 
This property, called the \textit{square-degree property}, enables us to apply results on graph limits of dense graphs to derive convergence and characterize the space of converging graph sequences.  
In particular, star graphs satisfy the square-degree property resulting in dense line graphs and non-zero graphons of line graphs. 
Similarly, superlinear preferential attachment graphs and some mixture graphs give rise to dense line graphs almost surely. In contrast, dense graphs, including Erdős–Rényi graphs make the line graphs sparse, resulting in the zero graphon.

\end{abstract}

\section{Introduction}\label{sec:intro}

In this paper we uncover a connection between line graphs and a subset of sparse graphs, that turn out to be useful for graph limits. Line graphs map edges to vertices and connects edges when edges in the original graph share a vertex. For a graph $G_n$ with $n$ nodes, a line graph $H_m := L(G_n)$
is a graph where each of the $m$ edges of the original graph $G_n$ is a node of $H_m$. We propose a new way to model sparse graphs by modeling the graphon
of the corresponding line graph. We show that sparse graphs $\{G_n\}_n$ that satisfy a certain property, which we call the ``square-degree property'' denoted by $Sq$,
have line graphs $H_m$ that are dense resulting in non-zero graphons.

A graphon is a symmetric, measurable function traditionally defined on the unit square as $W:[0,1]^2 \to [0,1]$; it also denotes the limit of a converging sequence of graphs. The limitation of graphons is that as a consequence of the Aldous-Hoover theorem \citep{aldous1981representations, hoover1979relations},  graphs generated from a graphon are either dense or empty (edgeless). As such, they cannot be used to model sparse graphs, which include most social, biological and molecular networks. However, graphons are a non-parametric modeling tool making them an attractive choice. Consequently, there is a rich literature on modeling sparse graphs extending the traditional  graphon framework (see Section \ref{sec:related-work}). 

Our work springs from the following observation: Disjoint star graph sequences are sparse, but their line graphs are disjoint clique graphs, which are dense and have non-zero graphons as illustrated in Figure \ref{fig:1to4stars}.  This is corroborated by a result from \citet{Janson2016321}, which states  line graph limits are disjoint clique graph limits. %Even though in the limit line graphs converge to disjoint cliques, graphs other than disjoint stars can have line graphs converging to disjoint cliques.  

\begin{figure}[!ht]
    \centering
    \includegraphics[width = 0.8\textwidth]{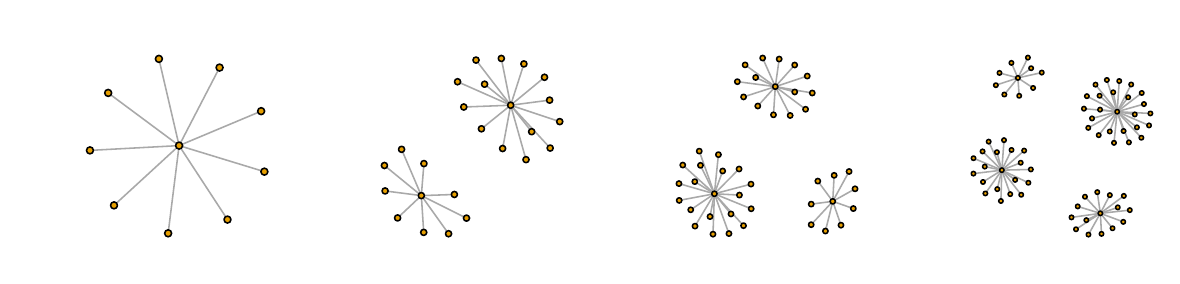}
    \includegraphics[width = 0.8\textwidth]{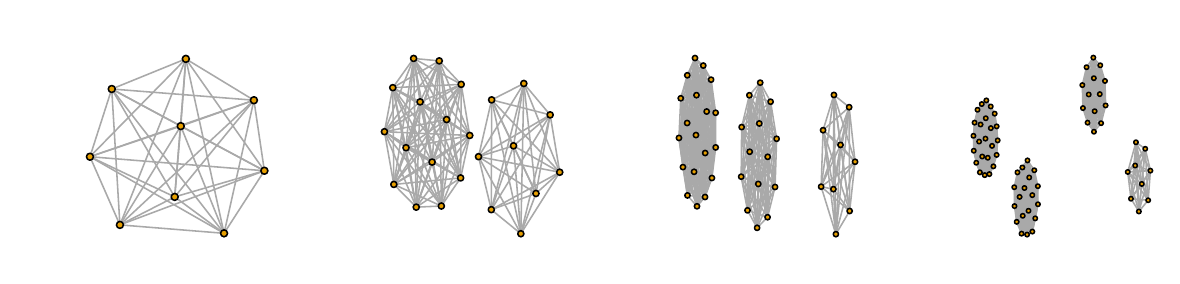}
    \includegraphics[width = 0.8\textwidth ]{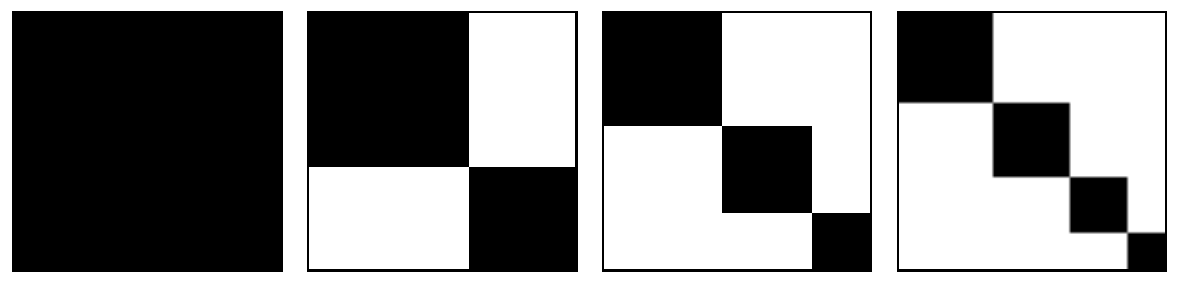}
    \caption{Top row: Graphs of 1 to 4 disjoint stars. Recall the graphon $W = 0$ for star graphs. Middle row: Line graphs of disjoint stars in top row. Line graphs of star graphs are complete graphs. Bottom row: The graphons of the line graphs of the star graphs (taken as sequences) shown on top. }
    \label{fig:1to4stars}
\end{figure}

\begin{figure}[!ht]
    \centering
    \begin{tikzpicture}[level 1/.style={sibling distance=8cm},
level 2/.style={sibling distance=5cm}, 
level 3/.style={sibling distance=4cm}, 
  every node/.style = {shape=rectangle, rounded corners,
    draw, align=center,
    top color=white, bottom color=black!20}]
  \node {$\{G_n\}_n$}
    child { node {Dense \\ $W \neq 0$} 
     child { node[text=blue]  (dense) { $\{G_n\}_n \notin S_q$ \\ $U= 0$ }  
     child {node[text=blue]  {e.g. Erdős–Rényi graphs \S\ref{sec:erdos-renyi}  }}} } 
    child { node {Sparse \\ $W=0$}
      child { node[text=blue] (sparseNotSq) {$ \{G_n\}_n \notin  S_q$ \\ $U = 0$} 
        child {node[text=blue] (paths)  {e.g. Paths $P_n$ \S\ref{sec:SparseNotInSq}, \\ Cycles $C_n$ \S\ref{sec:SparseNotInSq}}}} 
     child { node[text=blue] (Sq) {$ \{G_n\}_n \in S_q $ \\ $
      U \neq 0$ }
        child { node[text=blue]  {e.g. Superlinear PA graphs \S\ref{sec:prefattachment}, \\ Stars $K_{1,n}$ \S\ref{sec:stars}} } }
      };   
     % \node[below=of paths]{$(U,W)$-mixture graphs};
     \node[below=of paths, minimum width=10cm, text=blue] (UWmix) {$(U,W)$-mixture graphs \S\ref{sec:UWMixtures}};
     \draw[dashed] (dense) -- ($(dense)+(-3,0)$) -- ($(UWmix)+(-8.5,0)$) --   (UWmix);
     \draw[dashed] (sparseNotSq) -- ($(sparseNotSq)+(-2,0)$) -- ($(UWmix)+(-2,0.3)$);
    \draw[dashed] (Sq) -- ($(Sq)+(+3,0)$) -- ($(UWmix)+(+8,0)$) --  (UWmix);
    %\draw (nodeA) -- ($(nodeA)+(2,0)$) -- ($(nodeB)+(2,0)$) -- (nodeB);

    \end{tikzpicture}
    \caption{Characterization of graph sequences $\{G_n\}_n$ with results discussed in this paper in blue text. $\{G_n\}_n\in S_q$ indicates that the graph sequence satisfies the square-degree property. If $\{G_n\}_n$ converges to $W$, then for dense graphs $\{G_n\}_n$,  $W \neq 0$, but $U = 0$. Recall that sparse graphs converge to $W = 0$. For sparse $\{G_n\}_n \notin S_q$ we have $U = 0$. However if $\{G_n\}_n \in S_q$ and $\{H_m\}_m$ converges to $U$, then $U \neq 0$. Depending on the mixture properties $(U,W)$-mixture graphs can be dense, sparse or square-degree.}
    \label{fig:resultsdiagram}
\end{figure}

\citet{Janson2016321} showed the link between line graph limits and disjoint cliques; we use the fact that line graphs of stars are cliques to make the connection between stars (extremely sparse graphs) and dense graphons via the line graph operation. \citet{Janson2016321}  did not use line graphs to model sparse graphs, his focus was on hereditary properties of graphs and their limits. In addition to stars, we investigate other structures, both dense and sparse, such as Erdős–Rényi graphs, paths, cycles, preferential attachment graphs, stars, and the recently proposed $(U,W)$-mixture graphs \citep{kandOng2025graphonMixtures}.  

% \cheng{Jansson linked cliques to line graph limits, therefore we use the fact that line graphs of cliques are stars to show that star limits are graphons. He didn't use line graphs to model sparse graphs in general. We investigate the other structures, both dense and sparse graphs. We investigate structures such as Erdos Renyi graphs, paths, cycles, preferential attachment graphs, stars, and the recently proposed UW mixture graphons.}

Our main result (summarised in Figure~\ref{fig:resultsdiagram}) is a characterization of converging graph sequences.  The square-degree property $Sq$ is used in this characterization and blue text in Figure \ref{fig:resultsdiagram} denotes results obtained in this paper. This figure is on converging graph sequences $\{G_n\}_n$ and $\{H_m\}_m$ where $H_m$ denotes the line graph of $G_n$, i.e., $H_m = L(G_n)$.  We consistently use $W$ and $U$ for graphons where $G_n \to W$ and $H_m \to U$. It is well-known that for dense graph sequences $\{G_n\}_n$, the graphon $W \neq 0$, while sparse graph sequences correspond to $W = 0$. We show that dense graph sequences do not satisfy the square-degree property $Sq$ in Section~\ref{sec:starsquare}.  If $\{G_n\}_n$ is dense, then $\{H_m\}_m$ converges to $U = 0$, i.e., line graphs of dense graph sequences converge to the zero graphon.  
If $\{G_n\}_n$ is sparse then we know that $W = 0$.
However, we cannot distinguish between different sparse graphs using $W$. If $\{G_n\}_n$ satisfies $Sq$ and $H_m \to U$, then $U\neq 0$ and $U$ can distinguish different types of sparse graphs. 
% This means that line graphs of sparse graphs can be more revealing \cheng{I'm puzzled which section we want to refer to, and which $U$ we distinguish} which we illustrate in Section~\ref{sec:sq}. %and \ref{sec:probgraphs}.   
The square-degree property $Sq$ is important because only graphs satisfying $Sq$ give rise to $U \neq 0$. 
Furthermore, not all sparse graphs satisfy $Sq$. 
Paths $P_n$ or cycles $C_n$ are such sparse graphs. Therefore, the subset of sparse graphs satisfying $Sq$ gives us certain types of graphs such as stars $K_{1,n}$,  superlinear preferential attachment graphs or some $(U,W)$-mixture graphs \citep{kandOng2025graphonMixtures}.   For these graph sequences the line graphs converge to the limit $U \neq 0$.

We provide some background in Section~\ref{sec:notation}, 
and present our observations connecting sparse graphs $G_n$ with their line graphs $H_m$ in Section~\ref{sec:results}. 
In Section~\ref{sec:converging-sparse-graphs} we focus on converging graph sequences and show that graph sequences satisfying the max-degree condition (Definition \ref{def:maxdegreetoedges}) give rise to $U \neq 0$.
%We show in Section~\ref{sec:converging-sparse-graphs} 
%that graphs $G_n$ that satisfy the square-degree property with have convergent edge densities and homomorphism densities \cheng{Is this a good place to introduce max-degree?}.
We derive graphons of line graphs
for Erdos-Renyi graphs in Section~\ref{sec:dense} and explore sparse graphs that do not satisfy the squre-degree property in Section~\ref{sec:SparseNotInSq}.
We derive the graphons for disjoint star graphs and superlinear preferential graphs in Section~\ref{sec:sq}. As $(U,W)$-mixture graphs can be either dense or sparse depending on mixture properties, we delve into them in Section ~\ref{sec:UWMixtures}.

% \paragraph{Contributions of this paper}
% \begin{itemize}
%     \item We propose a property of sparse graphs, the square-degree property (Definition~\ref{def:square})
%     which allows us to find sparse graphs whose line graphs are dense.
%     In particular, sparse graphs with square-degree property 
%     have dense line graphs, and under certain conditions have line graph limits (Section~\ref{sec:havingconvergence}).
%     \item We prove that for disjoint star graphs, the corresponding line graphs are dense
%     and hence have graph limits (Section \ref{sec:resultsdeterministic}). Furthermore, we show that
%    certain preferential attachment graphs %and Erd\"{o}s–R\'{e}nyi graphs also have line graphs
%    have dense line graphs 
%     that converge to non-zero graphons under certain conditions (Section \ref{sec:prefattachment}).
%     \item We illustrate with empirical graphons the utility of line graphs for sparse graphs
%     in Section~\ref{sec:experiments}.
% \end{itemize}

\section{Notation and Preliminaries}\label{sec:notation}
A simple graph is a graph without loops or multiple edges between the same nodes. We only consider simple graphs and sequences of simple graphs in this paper. 

\subsection{Line graphs}\label{sec:notationlinegraphs}
Let $G$ denote a graph. If $G$ has at least one edge, then its line graph is the graph whose vertices are the edges of $G$, with two of these vertices being adjacent if the corresponding edges are adjacent in $G$ \citep{beineke2021line}. Figure \ref{fig:linegraph1} shows an example of a graph and its line graph.  The edges in the graph on the left are mapped to the vertices in the line graph (on the right) as can be seen from the numbers. 
 
\begin{figure}[!ht]
    \centering
    \includegraphics[width = \textwidth, trim={0cm 2cm 0 2cm},clip]{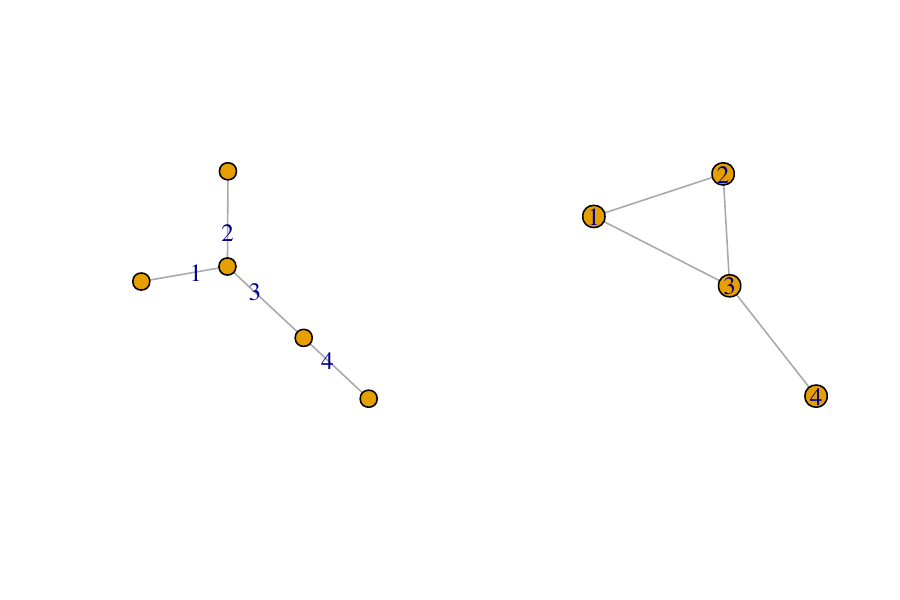}
    \caption{A graph on the left and its line graph on the right. }
    \label{fig:linegraph1}
\end{figure}

We denote the line graph operation by $L$, i.e., for a graph $G$ we denote its line graph by $H:= L(G)$. In terms of notation we make a distinction between graphs $G$ and line graphs $H$, i.e., we use the letter $H$, with and without subscripts, to denote line graphs. 

Rather than a single graph $G$, we are interested in graph sequences. The exact type of sequences which forms our interest will be made clear by the end of this section. Let $\{G_n\}_{n}$ %\cheng{I don't think we ever use this notation, but something else instead?} 
denote a graph sequence. The index $n$ mostly denotes the number of nodes in $G_n$ and let the number of edges be given by $m$. We denote the line graph of $G_n$ by $H_m := L(G_n)$ as $H_m$ has $m$ nodes. 

% \SK{I don't think this part is needed if going to Random Structures and Algorithms  . } 
% We use standard graph theory notation to denote specific types of graphs. As customary $K_n$ denotes a complete graph of $n$ nodes, and $K_{s,r}$ denotes a complete bi-partite graph of  partition sizes $s$ and $r$, i.e., there are $s$ nodes in one subset completely connected to $r$ nodes in the other subset. When $s = 1$ we get star graphs; $K_{1, n}$ denotes a star with $n + 1$ vertices, where $n$ vertices are connected to the hub vertex.

\begin{definition}\label{def:linegraphroot}
    If $G$ is a graph whose line graph is $H$, that is, $L(G) = H$, then $G$ is called the \textbf{root} of $H$. 
\end{definition}

From Whitney's work \citep{whitney1932congruent} it follows that the structure of a graph can be recovered from its line graph with one exception: if the line graph $H$ is $K_3$, a triangle, then the root of $H$ can be either  $K_{1,3}$, a star or $K_3$ a triangle. This follows from the following theorem as stated in \cite{harary2018graph}:

\begin{theorem}[Whitney1932, Harary 1969]
    Let $G$ and $G'$ be connected graphs with isomorphic line graphs. Then $G$ and $G'$ are isomorphic unless one is $K_3$ and the other is $K_{1,3}$.
\end{theorem}

Based on this theorem many algorithms compute the root of a line graph when $H$ is not $K_3$ \citep{ROUSSOPOULOS1973108, Lehot74, liu2015iligra}. 
%
%By simply creating edges corresponding to vertices in line graph $H$ and connecting them by merging the vertices if there is an edge between the vertices in $H$ we can obtain the the graph $G$, such that $H = L(G)$.   
% Thus, if $H$ is a line graph and it is not $K_3$, then we can talk about $L^{-1}(H)$. Many algorithms exist that compute the root of a line graph $H$ \citep{lin}
%
We state some preliminary results on line graphs covered in Chapter 1 of \cite{beineke2021line}. 
\begin{lemma}\label{lemma:linegraphs}
    Let $G$ be a non-null graph with $n$ vertices and $m$ edges. Let $H = L(G)$.  Then
    \begin{enumerate}
        \item $H$ has $m$  vertices and $\frac{1}{2}\sum (\deg\,  v)^2 - m$ edges  \label{lemma:linegraphs:numedges}
        \item If $G$ is an $r$-regular graph then $H$ is $2(r-1)$-regular and has $\frac{nr}{2}$ vertices. \label{lemma:linegraphs:r-regulargraph}
        \item If $G$ is a path $P_n$, then $H$ is also a path of $n-1$ vertices, i.e., $H = P_{n-1}$. \label{lemma:linegraphs:paths}
        \item If $G$ is a non-trivial connected graph, then $H$ is also connected. \label{lemma:linegraphs:connectedgraph}
        \item If $G$ is a cycle $C_n$ of $n$ vertices, then $H$ is also a cycle $C_n$ of $n$ vertices.  \label{lemma:linegraphs:cyle}
        \item If $G$ is a star, i.e., $G = K_{1, n-1}$, then $H$ is a complete graph of $n-1$ vertices, i.e. $H = K_{n-1}$. \label{lemma:linegraphs:star} 
      \end{enumerate}
\end{lemma}

\subsection{Graph sequences}\label{sec:notationgeneral}

The edge density of a graph $G$ with $n$ nodes and $m$ edges is given by $\density(G) = \frac{2m}{n(n-1)}$. Thus, from Lemma \ref{lemma:linegraphs}(1) the edge density of $H = L(G)$ is given by
\begin{equation}\label{eq:linedensity}
    \density(H) = \frac{\frac{1}{2}\sum (\deg\,  v^2) - m}{\frac{1}{2}m(m-1)} \, , 
\end{equation}
where $\deg v$ denotes the degree distribution of graph $G$ and $\deg v^2$ denotes the vector of squared degrees in $G$. We refer to the edge density simply as density. 

\begin{definition}[\bf Dense graph sequences]\label{def:dense}
    A sequence of graphs $\{G_n\}_n$ is dense if the number of edges $m$ grow quadratically with the number of nodes $n$, i.e., 
    $$ \liminf_{n \to \infty} \frac{m}{n^2} = c > 0 \, . %\density(G_n)
    $$ 
    We denote the set of all dense graph sequences by $D$.
\end{definition}
\begin{definition}[\bf Sparse graph sequences]\label{def:sparse}
     A sequence of graphs $\{G_n\}_n$ is sparse if the number of edges $m$ grow sub-quadratically with the number of nodes $n$, i.e., 
    $$ \lim_{n \to \infty} \frac{m}{n^2} = 0 \, . 
    $$
    We denote the set of all sparse graph sequences by $S$.
\end{definition}

For dense graph sequences, the density is bounded from below by a non-zero constant, whereas for sparse graph sequences it goes to zero. 
%Figure \ref{fig:venndensesparse} shows the Euler diagram of dense and sparse graph sequences. 
The density or $m/n^2$ of a sequence of dense graphs $\{G_n\}_n$ does not necessarily converge; the $\liminf$ is strictly positive, i.e., any converging subsequence has strictly positive density as $n \to \infty$. In contrast, the density or $m/n^2$  of sparse graphs converge to zero, i.e., the limit is equal to zero, not just the $\liminf$.  The set of dense graph sequences $D$ and the set of sparse graph sequences $S$ is non-intersecting. Furthermore, the complement of the union of $D$ and $S$, $ \overline{ D \cup S}$ is non-empty. It contains graph sequences $\{G_n\}_n$ such that $\liminf_{n \to \infty} m/n^2  =0 \neq \limsup_{n \to \infty} m/n^2 $, i.e, it is a mixture of dense and sparse graph sequences with the density of different subsequences converging to different limits with some converging to zero. 

\subsection{Graphons}\label{sec:notationgraphons}
Next we turn our attention to graphons. A graphon is a symmetric, measurable function $W:[0,1]^2 \rightarrow [0,1]$ often used to describe both the limiting properties of graph sequences as well as the graph generation process \citep{borgs2011limits}.  We define some terms often used in the graphon literature. 
% When defining graph limits, graph homomorphisms -- edge preserving maps from one graph to another --  are used. 

\begin{definition}\label{def:graphhomo}A \textbf{graph homomorphism} from $F$ to $G$ is a map $f : V(F) \to V(G)$ such that if $uv \in E(F)$ then $f(u)f(v) \in E(G)$. (Maps edges to edges.) Let $\Hom(F, G)$ be the set of all such homomorphisms and let $\hom(F, G) = |\Hom(F, G)|$. Then \textbf{homomorphism density} is defined as
$$ t(F, G) = \frac{\hom(F, G)}{ |V(G)|^{|V(F)|} } \, . 
$$ 
The number of homomorphisms  $\hom(F, G)$ is given by 
$$ \hom(F, G) = \sum_{\phi: V(F) \to V(G)} \prod_{uv \in E(F)} \beta_{\phi(u)\phi(v)}(G)
$$
where $\beta_{ij}(G)$ is the weight of edge $ij$ in graph $G$, which equals either 1 or 0 in unweighted graphs. 
For a graphon $W$, the homomorphism density is defined as
$$ t(F, W) = \int_{[0,1]^{|V(F)|}} \prod_{ij \in E(F)} W(x_i, x_j) \,   dx \, . 
$$
\end{definition}
A graph homomorphism is an edge preserving map from one graph to another. The homomorphism density is useful as it is bounded even when the number of homomorphisms $\hom(F, G)$ go to infinity.

\begin{definition}\label{def:cut1}The \textbf{cut norm} of graphon $W$   \citep{frieze1999quick, borgs2008convergent} is defined as
$$ \lVert W \rVert_\square = \sup_{S, T} \left\vert\int_{S\times T} W(x,y) \, dx dy \right\vert  \, , 
$$
where the supremum is taken over all measurable sets $S$ and $T$ of $[0 ,1]$.  \end{definition}

\begin{definition}\label{def:cut2} Given two graphons $W_1$ and $W_2$ the \textbf{cut metric}  \citep{borgs2008convergent} is defined as 
$$ \delta_{\square}(W_1, W_2) = \inf_\varphi \left\lVert W_1 - W_2^\varphi \right\rVert_\square \, , 
$$
where the infimum is taken over all measure preserving bijections $\varphi:[0,1] \rightarrow [0,1]$.  
\end{definition}
Let $\mathcal{W}$ denote the space of graphons, i.e., $\mathcal{W} = \{ W \in \mathcal{W} \}$. Then, the cut metric is a pseudo-metric in $\mathcal{W}$ because $\delta_{\square}(W_1, W_2)  = 0 $ does not imply $W_1 = W_2$, i.e., $ \delta_{\square}(W_1, W_2)  \geq 0 $ for $W_1 \neq W_2$. However the cut metric $\delta_{\square} $  is a metric on the quotient space $\tilde{\mathcal{W}} = \mathcal{W}/ \sim  $ where $f \sim g$  if $f(x, y) = g(\sigma x,  \sigma y)$ for some measure preserving $\sigma$.

%\cheng{Is $t$ the standard notation for this? If not, perhaps use something else, just in case we ever write a paper about time $t$.}\SK{I think $t$ is the standard notation for this, I've seen $t$ been used by all the Borgs-Chayes papers and other papers as well.}
%\cheng{I don't see $t$ being used anywhere.}\SK{This will change.}

\begin{definition}\label{def:wrandomgraphs}Uniformly pick $x_1, x_2, \ldots x_n$ from $[0,1]$. A \textbf{W-random graph} $\mathbb{G}(n,W)$ has the vertex set $1, 2, \ldots n$ and vertices $i$ and $j$ are connected with probability $W(x_i, x_j)$.
\end{definition}
We can think of $W$-random graphs as graphs sampled from the graphon $W$.  

The homomorphism density is used to define graph convergence. 

\begin{definition}[\citep{borgs2008convergent}]\label{def:convergent}
A graph sequence $\{G_n\}_n$ is said to be convergent if  $t(F, G_n)$ converges as $n$ goes to infinity for any simple graph $F$. 
\end{definition}

Every finite, simple graph $G$ can be represented by a graphon $W_G$, which we call its empirical graphon. 
\begin{definition}\label{def:empiricalgraphon}
    Given a graph $G$ with $n$ vertices labeled $\{1, \ldots, n\}$, we define its \textbf{empirical graphon} $W_G: [0, 1]^2 \rightarrow [0, 1]$ as follows: We split the interval $[0, 1]$ into $n$ equal intervals $\{I_1, I_2, \ldots, I_n \}$ (first one closed, all others half open) and for $x \in I_i, y \in I_j$ define
    $$ W_G(x,y) = \begin{cases}
        1 & \, \text{if} \quad  ij \in E(G) \, \\
        0 & \, \text{otherwise} \, ,
    \end{cases}
    $$
\end{definition}
where $E(G)$ denotes the edges of $G$. The empirical graphon replaces the the adjacency matrix with a unit square and the $(i,j)$th entry of the adjacency matrix is replaced with a square of size $(1/n) \times (1/n)$.  

The cut metric between graphs $G$  and $G'$ is defined as $\delta_\square\left(G, G'\right) = \delta_\square\left(W_G, W_{G'} \right) $. The cut metric between a graph $G$ and a graphon $U$ is defined as $\delta_\square\left(G, U \right) = \delta_\square\left(W_G, U \right) $.

The following theorem show that a graph sequence converges to a graphon when the corresponding homomorphism density converges. The subsequent theorem establishes the equivalence between homomorphism density and the cut metric convergence.
%\cite{borgs2008convergent} prove the following theorem for convergent graph sequences. 

\begin{theorem}[\cite{borgs2008convergent}]\label{thm:BCconvergent} For every convergent sequence $\{G_n\}_n$ of simple graphs there is a graphon $W$ with values in $[0,1]$ such that $t(F, G_n) \rightarrow t(F, W)$ for every simple graph $F$. Moreover for every graphon $W$ with values in $[0,1]$ there is a convergent sequence of graphs satisfying this relation. 
\end{theorem}

\begin{theorem}[\cite{borgs2011limits}]\label{thm:BCconvergent2} 
A sequence of graphs $\{G_n\}_n$ is convergent if and only if it is Cauchy in the $\delta_\square$ distance. The sequence  $\{G_n\}_n$ converges to $W$ if and only if $\delta_\square(W_{G_n}, W) \to 0$. Furthermore, if this is the case, and $|V(G_n)| \to \infty$, then there is a way to label the nodes of the graphs $G_n$ such that $\lVert W_{G_n}  - W \rVert_\square \to 0$. 
\end{theorem}

\subsection{Line graph limits and disjoint unions}\label{sec:linegraphlimits}

Next we state a result in \citet{Janson2016321} on line graph limits, which is quite relevant to this work. 
\cite{Janson2016321} used a real-valued sequence called the \textit{mass partition} to represent disjoint clique graph limits. 

\begin{definition}\label{def:masspartitionJanson}(\textbf{Mass-partition})
    Let a mass-partition to be a sequence $\bm{p} = \{p_i\}_{i = 1}^{\infty}$ of non-negative real numbers such that
    \[ p_1 \geq p_2 \geq \cdots \geq 0 \quad \text{and} \quad \sum_{i = 1}^\infty p_i \leq 1 \, . 
    \]
    Let $\mathcal{M}$ be the set of all mass-partitions. 
\end{definition}
% For every mass-partition we can define a graphon as follows:

\citet{Janson2016321} defined the graphon corresponding to a mass-partition $p$ as follows.

\begin{definition}\label{def:masspartitiongraphon}
  Given a mass-partition $\bm{p}$ consider disjoint subsets $\{A_i\}_{i = 1}^\infty$ of the probability space $( [0,1], \mu)$ such that $\mu(A_i) = p_i$ where $\mu$ denotes the Lebesgue measure. Then the graphon corresponding to the mass-partition $\bm{p}$ is defined as $W_{\bm{p}}^\mathcal{M} = \sum_{i = 1}^{\infty} \bm{1}_{A_i \times A_i}$. 
\end{definition}
% Given a mass-partition $\bm{p}$ (Definition \ref{def:masspartitionJanson}) \cite{Janson2016321} defined a graphon $W_{\bm{p}}^\mathcal{M}$ by taking disjoint subsets $\{A_i\}_{i = 1}^\infty$ of a probability space $(S, \mu)$ such that $\mu(A_i) = p_i$ and defining $W_{\bm{p}}^\mathcal{M} = \sum_{i = 1}^{\infty} \bm{1}_{A_i \times A_i}$. 

% \begin{definition}\label{def:masspartitiongraphon}
% Given a mass-partition $\bm{p}$ (Definition \ref{def:masspartitionJanson}) we define a graphon $W_{\bm{p}}^\mathcal{M}$ by taking disjoint subsets $\{A_i\}_{i = 1}^\infty$ of a probability space $(S, \mu)$ such that $\mu(A_i) = p_i$ and defining $W_{\bm{p}}^\mathcal{M} = \sum_{i = 1}^{\infty} \bm{1}_{A_i \times A_i}$. 
% \end{definition}

% We can define a mass-partition for any graph. To do that we use component sizes. 
A mass-partition can be defined for any graph using the sizes of its connected components. 
\begin{definition}\label{def:compsize}
    For any graph $G$, we denote the component sizes by $C_1(G) \geq C_2(G) \geq \cdots$, ordered such that when $k$ is larger than the number of components $C_k(G) = 0$.  
\end{definition}
Each graph $G$ defines a mass-partition $\{C_i(G)/|G|\}_{i=1}^\infty$ where $|G|$ denotes the number of nodes in $|G|$.
The next theorem links mass partitions, graphons and disjoint clique graphs. 
\begin{theorem}\label{thm:Janson1} {\bf(\cite{Janson2016321} Thm 7.5) } If $\{G_n\}_n$ is a sequence of disjoint clique graphs with $|G_n| \to \infty$ and $W$ is a graphon, then $G_n \to W$ if an only if $W = W_{\bm{p}}^\mathcal{M}$  for some $\bm{p} \in \mathcal{M}$ and $\{ \left(C_i(G_n)/|G_n| \right)_i \} \to \bm{p}$.
\end{theorem}

Theorem \ref{thm:Janson1} is on converging disjoint clique graph sequences  $\{G_n\}_n$. It tells us the component sizes converge to the mass-partition and that there is a one-to-one correspondence between the disjoint clique graphon and its mass-partitions.
\begin{theorem}\label{thm:Jansonmain} {\bf(\cite{Janson2016321} Thm 8.3) }
 A graph limit is a line graph limit if and only if it is a disjoint clique graph limit.     
\end{theorem}
\citet{Janson2016321} showed that line graph limits are disjoint clique graph limits and gave an explicit formulation of disjoint clique graph limits in terms of mass-partitions. But the link to modeling sparse graphs via the line graph was not made. We add this link and extend Janson's Theorem \ref{thm:Janson1} by exploring other graphs $\{G_n\}_n$ for which the line graphs converge to disjoint clique graphons. 

The following theorems and definition from \cite{janson2008connectedness} focus on disjoint graph unions and the corresponding graphon.  %\cheng{We should refer to these theorems when we use them. I cannot seem to find any reference to them.}

\begin{theorem}\label{thm:jansonAddedGraph}{\bf(\cite{janson2008connectedness} Thm 4.2(i))}
    Let $W_1$, $W_2$ be graphons, $F$ a connected, simple graph  and $0 \leq \alpha \leq 1$. There exists a unique $W$ (up to bijections) such that
    \begin{equation}\label{eq:jansonGraphSum}
         t(F, W) =   \alpha^{|V(F)|} t(F, W_1) +  (1 - \alpha)^{|V(F)|} t(F, W_2) \, . 
    \end{equation}
\end{theorem}

Theorem \ref{thm:jansonAddedGraph} describes how homomorphism densities of graphons $W_1$ and $W_2$ combine to form a new graphon $W$. 

\begin{definition}\label{def:oplus}
    Let $W_1$ and $W_2$ denote graphons and $0 \leq \alpha \leq 1$. Then the graphon determined by equation \eqref{eq:jansonGraphSum} is denoted by $\alpha W_1 \oplus (1-\alpha) W_2$.
\end{definition}

\begin{theorem}\label{thm:jansonGraphonAddition}{\bf( \cite{janson2008connectedness} Thms 1.1 and 1.8)}
    Suppose $\{G_n\}_n$ and $\{G'_n\}_n$ converge to $W_1$ and $W_2$ respectively. If further $|V(G_n)|/(|V(G_n)| + |V(G'_n)| ) \to \alpha \in [0, 1] $ then $G_n \cup G'_n \to \alpha W_1 \oplus (1- \alpha)W_2 $. 
\end{theorem}

Graphon $\alpha W_1 \oplus (1- \alpha)W_2$ can be visualized by placing a scaled graphon $W_1$ on $[0, \alpha]^2$  and a scaled $W_2$ on $[\alpha, 1]^2$ as illustrated in Figure \ref{fig:graphonAddition}. 

\begin{figure}[!t]
    \centering
    \includegraphics[width=0.9\linewidth]{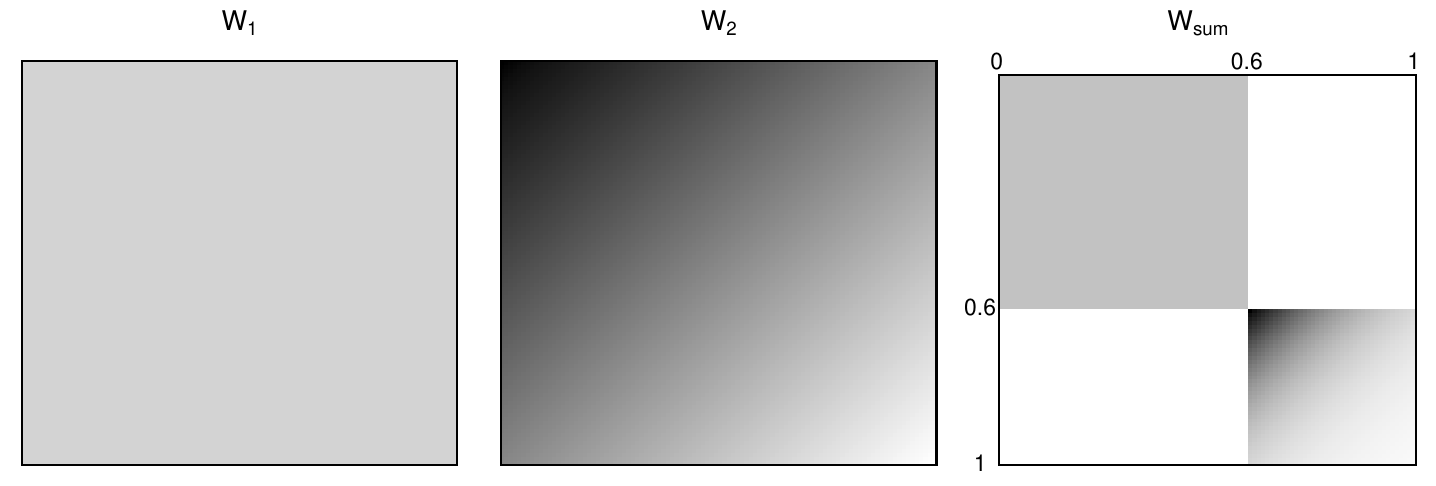}
    \caption{Graphons $W_1$, $W_2$ and $W_{\text{sum}} = 0.6 W_1 \oplus 0.4W_2$.}
    \label{fig:graphonAddition}
\end{figure}

\subsection{Related work}
\label{sec:related-work}

\subsubsection{Graphons of sparse graphs}

The fundamental problem of modeling sparse graphs using graphons is that the graphon converges to zero for sparse graphs. Collaborations led by Borgs and Chayes have resulted in considerable work on sparse graph limits \citep{Borgs2017LD,borgs2019Lp, borgs2018sparse}. Their work encompasses a variety of analytic and measure theoretic solutions including new norms and convergence measures.  \citet{borgs2018sparse} overcome the zero graphon by stretching the domain of the graphon to the whole positive quadrant (stretched graphon) in one formulation and rescaling the graphon values to the positive reals (rescaled graphon) in the other. 

It is worth noting that traditional graphons model graphs as discrete structures using the theoretical underpinnings of exchangeable random arrays \citep{diaconis2008graph} and the Aldous-Hoover theorem. %Roughly stated, the Aldous-Hoover theorem implies a mixture model representation for infinite exchangeable binary matrices. For undirected graphs this mixture model is the graphon. 
\cite{caron2017sparse} demonstrate that treating graphs as continuous structures makes it possible to represent sparse graphs using a graphon based on Kallenberg exchangeability. %show that by modeling graphs as continuous structures enables them to model sparse graphs using a graphon representation based on Kallenberg exchangeability. 
They use point processes to model graphs and use exchangeability theorems for continuous domains in their regularity framework. This construction was later used by \cite{veitch2015class} to introduce graphexes $\mathcal{W} = (W, S, I)$, a triple consisting of a graphon, a star function and isolated edges. Graphexes can model sparse graphs and was extended to multi-graphexes in \cite{borgs2021limits}.

Edge-exchangeability is another avenue used to model sparse graphs. The theory of exchangeable random matrices is generally applied to the adjacency matrix and thus consider vertex exchangeable graphs, which are invariant under node permutations. If the node id is indicative of the time a node joins the graph, vertex exchangeability implies that the connections a node forms is independent of the time it joins. But some sparse graphs, such as preferential attachment graphs do not obey this rule and thus are not vertex exchangeable. For these kinds of graphs, the earlier a node joins, the more connections it makes.  Instead of considering vertex exchangeability, \cite{crane2018edge, Crane_Dempsey_2019} introduce edge-exchangeable network models and show that these models allow for sparse structure and power-law degree distributions. \cite{cai2016edge} and \cite{Janson2018} extend this model 
show that graphs ranging from dense to very sparse graphs can be generated by using their construction. 

Another way to model sparse graphs is to consider a real-valued sequence $\{\rho_n\}_n$ converging to zero and to consider a sequence of graphons $\{W_n\}_n$ where $W_n = \rho_n W$.  As $W_n \to 0$, graphs sampled from $\{W_n\}_n$ are sparse. While graphs sampled from $W_n$ do not exhibit hub structure, they are somewhat similar to  Erdős–Rényi  $G(n, p)$ graphs with $p = \rho_n$ .  

Our work has similarities and differences to the work discussed above. The key difference in our work is that instead of working with the original graph we work with a transformed graph, viz., its line graph. We show that a subset of sparse graphs are dense under the line graph transformation. As vertex permutation in the line graph space is the same as edge permutation in the original graph space, we expect graphs corresponding to dense line graphs to be edge exchangeable. Thus, this work has some overlap with edge-exchangeability discussed above.   

\subsubsection{Line graphs and edge exchangeability}\label{sec:edge} 
Figure \ref{fig:edgeexchange} shows the connection between vertex and edge exchangeability when we map from graphs to line graphs. Graph $G$ is shown on the top left and its line graph $H = L(G)$ is shown on the top right. The graph on the bottom right $H'$ is $H$ with vertices permuted. Let us call the graph on the bottom left $G'$. Following definition \ref{def:linegraphroot} we can see that $G'$ is the root of $H'$ , i.e., $H' = L(G')$.  Furthermore, the vertex permutation  $\phi$ relabeled the vertices $(1, 2, 3, 4)$ in $H$ to $(2, 3, 4, 1)$ in $H'$. We see the same permutation occurs in edges from $G$ to $G'$, i.e. $G'$ is an edge permuted version of $G$. This is not surprising as line graphs map edges to vertices. 

\begin{figure}[!ht]
    \centering
    \includegraphics[width=0.7\textwidth]{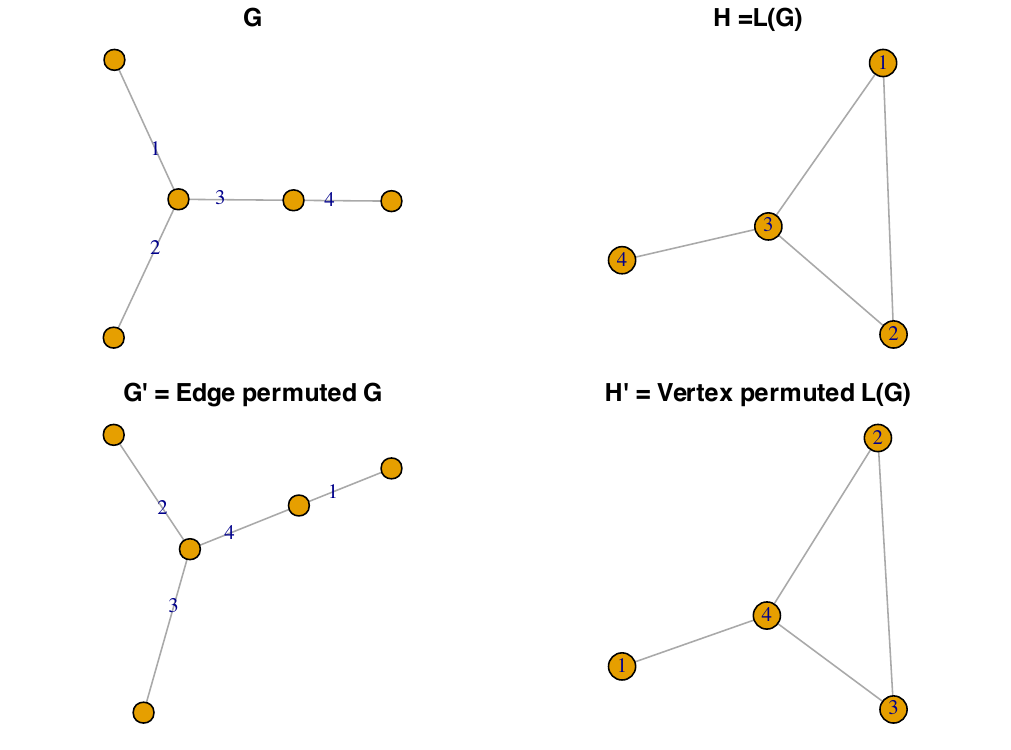}
    \caption{Vertex and edge exchangeability in graphs and line graphs. Graphs $G$ and $H = L(G)$ on the top row. Graph $H'$ is a vertex permuted version of $H$. We see that $H' = L(G')$, where $G'$ is the edge permuted version of $G$. }
    \label{fig:edgeexchange}
\end{figure}

\subsubsection{Graphon applications in machine learning}
The theory of graphons is well-developed and is used in machine learning applications for modeling and explanation \citep{sabanayagam2022graphon, saha2024graphonexplainer}. For example, it is desirable for a machine learning model to be transferable. \cite{RuizGraphonNN}  proposed graphon neural networks as the limit of graph neural networks (GNNs) with the aim of producing transferable GNNs. They showed that GNNs are transferable between deterministic graphs obtained from the same graphon. Graphons and the associated theory is used to bolster theoretical aspects of other topics. \cite{Levie2023} proposed a graph signal similarity measure for message passing neural networks based on the graphon cut distance. Hence they extended the cut distance to graph signals. Graph embeddings are used for a myriad of downstream tasks such as node classification, clustering and link prediction. \cite{Davison2024} investigated theoretical aspects of graph embeddings and showed that embedding methods implicitly fit graphon models. Under the assumption the graph is exchangeable, they described the limiting distribution of embeddings learned via subsampling the network. Graph homomorphisms are closely connected to graphons.  \cite{hanbaek2023} introduced motif sampling, which is essentially sampling graph homomorphisms uniformly at random.  %They proposed two MCMC algorithms for sampling random graph homomorphisms. 
\cite{asiri2026flowette} proposed Flowette, a graphon based flow-matching framework for graph generation. They proposed a new object called the \textit{graphette} that can sparsify graphons using a graph modification function and used a flow-matching framework to generate sparse graphs such as molecular graphs and trees. %  \SK{to finish}

% ------------------------------------------------------------------
\section{Sparse graphs with dense line graphs}\label{sec:results}
% -----------------------------------------------------------------

In this section, we show that there are sparse graphs whose line graphs are dense. 
In particular we show in Theorem~\ref{thm:main} that sparse graphs with square-degree property
(Definition~\ref{def:square})
have corresponding line graphs that are dense, and vice versa.
%We show in Section~\ref{sec:havingconvergence} that under certain conditions, 
%the corresponding line graphs converge with respect to the homomorphism density,
%leading to graphons of line graphs. 
This enables us to define a novel approach
to defining graph limits for sparse graphs by their associated line graphs.
Recall we denote graph sequences as $\{G_n\}_n$ and the corresponding line graph sequence as $\{H_m\}_m$. If the sequences converge, then we consistently use $W$ and $U$ for graphons corresponding to $\{G_n\}_n$ and $\{H_m\}_m$ respectively.
Next we define a property of a graph sequence that we call the \textit{square-degree property }. 
\begin{definition}[\bf Square-degree property $Sq$]\label{def:square} Let $\{G_n\}_n$ denote a sequence of graphs.  We say that $\{G_n\}_n$ exhibits the square-degree property if there exists some $c_1  > 0 $ and $N_0 \in \mathbb{N}$ such that for all $n \geq N_0$ we have
$$ \sum \deg v_{i, n}^2 \geq c_1 \left( \sum \deg v_{i, n} \right)^2  \, . 
$$  
We denote the set of graph sequences satisfying the square-degree property by $S_q$, i.e. if $\{G_n\}_n$ satisfies $Sq$ then $\{G_n\}_n \in S_q$.
\end{definition}

We note that Cauchy-Schwarz inequality gives $c_1 = 1/n$, which is not satisfactory as we need a strictly positive lower bound $c_1 > 0$ for all $n$. 
The square-degree property says that the ratio between the sum of the degree squared 
and square of the sum of degrees is bounded from below as $n$ goes to infinity. 
As the degree of a node is either zero or positive, this cannot be satisfied if the degree distribution is uniform, because then the sum of the mixed product terms $\deg v_{i, n} \times \deg v_{j,n}$ would hold the bulk weight compared to the  square terms $(\deg v_{i, n})^2$, especially as there are $n \choose{2}$ mixed product terms and only $n$ square terms. Therefore, we expect a graph sequence satisfying this property to have some inequalities in the degree distribution. For example, it may contain a set of ``big player'' nodes with large degree values. 

\citet{kandOng2025graphonMixtures} showed that the square-degree property is equivalent to a simpler condition called \textit{max-degree} for graph sequences $\{G_n\}_n$ that have converging line graphs.   

% Copied
\begin{definition}{\bf (Max-degree condition)}\label{def:maxdegreetoedges}
    Let $\{G_{n}\}_n$ be a sequence of graphs where $G_{n}$ has $n$ nodes and $m$ edges. Let the maximum degree of $G_{n}$ be denoted by $d_{\max, n}$. We say the $\{G_{n}\}_n$ satisfies the max-degree condition if there exists $c > 0$ such that
    % and $N_0 \in \mathbb{N}$ such that
    % \[  \frac{d_{\max, n}}{m} \geq c  > 0\,  
    % \]
    % for all $n > N_0$. Alternatively, 
    \[ \liminf_{n, m \to \infty} \frac{d_{\max, n}}{m} = c > 0 \, . 
    \]
    Let $S_x$ denote the set of graph sequences satisfying the max-degree condition. 
\end{definition}

\begin{lemma}[\cite{kandOng2025graphonMixtures} Lemma 2.6)]\label{lemma:squareandmaxequivlence}
    Let $\{G_n\}_n$ be a graph sequence with $H_m = L(G_n)$. Suppose $\{H_m\}_m$ converges to $U$.  Then
    \[  \{G_n\}_n \in S_q \equiv \{G_n\}_n \in S_x \, . 
    \]
\end{lemma}

This was proved recognizing that line graph limits can be represented by mass-partitions. 
The focus of \citet{kandOng2025graphonMixtures} was on graphon mixtures, their statistical properties and applicability to machine learning, whereas in this paper we investigate dense and sparse graph sequences and characterize them as illustrated in Figure \ref{fig:resultsdiagram}.

\subsection{Graph sequences with square-degree property $Sq$ are sparse}\label{sec:starsquare}

In sections \ref{sec:starsquare} and \ref{sec:linegraphs1} we continue with square-degree property $Sq$ and consider graph sequences $\{G_n\}_n$ that does not need to converge in homomorphism density. 

% \begin{restatable}{lem}{primelemma}
% \label{mylemma}
% Let $p$ be a prime number, and assume $p$ divides the product of two integers $a$ and $b$. Then $p$ divides $a$ or $p$ divides $b$.
% \end{restatable}

\begin{restatable}{lemma}{lemmapropsparse}\label{lemmapropsparse}
    If $\{G_n\}_n \in S_q \Longrightarrow \{G_n\}_n \in S$, i.e., graph sequences satisfying the square-degree property are sparse.
\end{restatable}
\begin{proof}
As $\{G_n\}_n \in S_q$ there exist some $c_1 > 0$ and $N_0 \in \mathbb{N}$ such that for all $n \geq N_0$ we have
$$ \sum \deg v_{i, n}^2 \geq c_1 \left( \sum \deg v_{i, n} \right)^2  \, . 
$$  
As 
\begin{equation}\label{eq:sq2}
     n(n-1)^2 \geq \sum \deg v_{i, n}^2 \geq c_1 \left( \sum \deg v_{i, n} \right)^2  = 4 c_1 m^2 \, , 
\end{equation}
we get $m \in O(n^{3/2})$ making $\{G_n\}_n$ sparse. From the above inequality we can see that 
$$ \limsup_{n \to \infty} \frac{m^2}{n^4} = \limsup_{n \to \infty} \frac{1}{4c_1 n} =   0 \, , 
$$
making $\lim_{n \to \infty} m/n^2 = 0$.
\end{proof}

% \begin{proof}
% As $\{G_n\}_n \in S_q$ there exist some $c_1 > 0$ and $N_0 \in \mathbb{N}$ such that for all $n \geq N_0$ we have
% $$ \sum \deg v_{i, n}^2 \geq c_1 \left( \sum \deg v_{i, n} \right)^2  \, . 
% $$  
% As 
% \begin{equation}\label{eq:sq2}
%      n(n-1)^2 \geq \sum \deg v_{i, n}^2 \geq c_1 \left( \sum \deg v_{i, n} \right)^2  = 4 c_1 m^2 \, , 
% \end{equation}
% we get $m \in O(n^{3/2})$ making $\{G_n\}_n$ sparse. From the above inequality we can see that 
% $$ \limsup_{n \to \infty} \frac{m^2}{n^4} = \limsup_{n \to \infty} \frac{1}{4c_1 n} =   0 \, , 
% $$
% making $\lim_{n \to \infty} m/n^2 = 0$.
% \end{proof}

Lemma \ref{lemmapropsparse} shows that the sparse graphs are a superset of graphs satisfying the square-degree property. However, not all sparse graphs satisfy $Sq$, for example paths and cycles. Therefore
$$ S_q \subset S \, . 
$$

\begin{corollary} If $\{G_n\}_n \in D \Longrightarrow \{G_n\}_n \notin S_q$, i.e., 
dense graph sequences do not satisfy the square-degree property.     
\end{corollary}
\begin{proof} 
As $D \subset \bar{S}$, where $\bar{S}$ denotes the complement of $S$, this is true because of the contrapositive of Lemma \ref{lemmapropsparse}. It can be quickly verified that dense graph sequences do  not satisfy  equation \eqref{eq:sq2} in  Lemma \ref{lemmapropsparse} because for dense graphs  $m  \geq c n^2$ for some $c>0$.
\end{proof}

\subsection{Only line graphs of graphs with square-degree property are dense}\label{sec:linegraphs1}
% \cheng{Style question: I wonder whether there is a difference between saying \\ (1) $\{G_n\}_n\in S_q \Longrightarrow \{H_m\}_m\in D$, and $\{H_m\}_m\in D \Longrightarrow \{G_n\}_n\in S_q$ \\ (2) $\{G_n\}_n\in S_q \equiv \{H_m\}_m \in D$.} \SK{I have used option 2. But I don't see it often in math papers.}

\begin{theorem}\label{thm:main}
 Let $\{G_n\}_n\in S$ be a sparse graph sequence. Let $\{H_m\}_m$ be the corresponding sequence of line graphs with $H_m = L(G_n)$. Then $\{G_n\}_n\in S_q \equiv \{H_m\}_m \in D$, i.e., $\{G_n\}_n$ satisfies $Sq$ if and only if $\{H_m\}_m$ is dense. 
 \end{theorem}
%  we have the following:
%  \begin{enumerate}
%      \item If $\{G_n\}_n$ satisfies  the square-degree property, that is $\{G_n\}_n\in S_q$, then the sequence of line graphs $\{H_m\}_m$ are dense, that is, 
% $$ \liminf_{m \to \infty} \density( H_m) = c >0 \, .
% $$
%indicating that the non-zero elements of the line-graphon $W_\ell$ of $\{G_n\}_n$ have positive area.
%     \item If the sequence of line graphs  $\{H_m\}_m$ are dense, then $\{G_n\}_n$ satisfies the square-degree property, that is $\{G_n\}_n\in S_q$.    
%  \end{enumerate}
% \end{theorem}
% Suppose $\{G_n\}_n$ denotes a graph sequence that satisfies the star property or the square-degree property. Then the line graphs of $\{G_n\}_n$ are dense, i.e., 
% $$ \lim_{n \to \infty} \density( L(G_n)) = c >0 \, .
% $$

\begin{proof}
\begin{enumerate}
    \item First we show $\{G_n\}_n\in S_q \Longrightarrow \{H_m\}_m\in D$. Suppose  $\{G_n\}_n \in S_q$ . Then from Definition \ref{def:square} there exists some $c_1 > 0$ and $N_0 \in \mathbb{N}$ such that for all $n \geq N_0$ we have
$$ \sum \deg v_{i, n}^2 \geq c_1 \left( \sum \deg v_{i, n} \right)^2  = 4 c_1 m^2 \, ,
$$     
where $m$ denotes the number of edges in $G_n$. From equation~\eqref{eq:linedensity} the edge density of the line graph $L(G_n)$ is
\begin{align}
    \density(H_m) & = \frac{\frac{1}{2}\sum_i (\deg\,  v_{i,n})^2 - m}{\frac{1}{2}m(m-1)} \, ,  \\
    & \geq \frac{\frac{1}{2}4c_1m^2 - m }{\frac{1}{2}m(m-1)} \, ,  \\
    & = \frac{2c_1 - \frac{1}{m}}{\frac{1}{2} -\frac{1}{2m} } \, .
\end{align}
Thus,
$$ \liminf_{m \to \infty} \density(H_m) = 4c_1 > 0 \, . 
$$
% Note that if $\{G_n\}_n$ satisfies the star property this holds true as the star property implies the square-degree property.  %Even in cases where $m$ is bounded we get line graphs with non-zero density. When the line-graphs have positive density as $n \to \infty$, the associated line-graphon has positive area contributed by the non-zero elements of the line-graphon.

  \item Next we show $\{H_m\}_m\in D \Longrightarrow \{G_n\}_n\in S_q$. If the line graphs $\{H_m\}_m$ are dense, i.e., $\{H_m\}_m \in D$ we have
  \begin{align}
    \density(H_m) & = \frac{\frac{1}{2}\sum_i (\deg\,  v_{i,n})^2 - m}{\frac{1}{2}m(m-1)} \geq c > 0  \quad \text{for all} \quad m > M_0 \in \mathbb{N}. 
\end{align}
This can only happen when 
$$ \sum_i (\deg\,  v_{i,n})^2   \geq c'm^2 \quad \text{where} \quad c' >0\, ,   \\
$$
implying that $\{G_n\}_n$ satisfies the square-degree property. 

% \begin{align*}
%     \text{As} \quad n-1 & \geq \deg\,  v_{i,n} \, ,  \\
%    \text{we have} \quad n(n-1)^2 & \geq \sum_i (\deg\,  v_{i,n})^2   \geq c'm^2  \, , \\
%   \text{giving us} \quad   \frac{1}{c'}\sqrt{n}(n-1) & \geq m \, .
% \end{align*}
% Thus, $m$, the number of edges in $G_n$ grow sub-quadratically making $\{G_n\}_n$ sparse. 
\end{enumerate}
% In both cases $\{G_n\}_n$ is sparse from Lemma \ref{lemmapropsparse}. 
\end{proof}

% Figure \ref{fig:vennstarsquaresparse} illustrates the results of Theorem \ref{thm:main}, where the shaded region denotes dense line graphs and the non-shaded region denotes sparse line graphs. The graph sequence $\{G_n\}_n$ satisfies $Sq$, i.e., $\{G_n\}_n \in S_q$ if and only if  the line graphs $\{H_m\}_m$ are dense. 

% \begin{figure}[!ht]
%     \centering
%     \begin{tikzpicture}
%         \draw (2,2) ellipse (4cm and 1.5cm);
%         \node[] at (2.5,3.75) {$S = \{ \text{sparse} \, \{G_n\}_n \}$ };
%         \draw[fill=gray!30] (2,2) ellipse (2cm and 1cm);
%         \node[] at (2,2) { $S_q$};
%         % \draw (1,2) ellipse (1.5cm and 1cm);
%         %\node[] at (2.5,2.5) { $St$};
%         % \node[] at (1,2) { $S_t$};
%         \node[] at (4.75,1.5) { $\{P_n\}_n$};
%         \node[] at (5,2.5) { $\{C_n\}_n$};
%     \end{tikzpicture}
%     \caption{An Euler diagram of the space of sparse graphs $S$ with $S_q \subset S$. The set $S\backslash S_q$ is non-empty as paths $\{P_n\}_n$, cycles $\{C_n\}_n$ and other graphs live here. The line graphs $\{H_m\}_m$ are dense in the shaded set  $S_q$.    }
%     \label{fig:vennstarsquaresparse}
% \end{figure}

Next we explore graph sequences $\{G_n \}_n$ that do not satisfy $Sq$, i.e. $\{G_n \}_n \notin S_q$.

\begin{lemma}\label{lemma:notsqu}
    If $\{G_n \}_n$ does not satisfy the square-degree property, i.e., $\{G_n \}_n \notin S_q$, then
% there exists a subsequence $\{G_{n'} \}_{n'}$ such that $H_{m'} = L(G_{n'})$ and
    %      $$ \liminf_{m' \to \infty}\density(H_{m'}) = 0 \, . 
    % $$
    $$ \normalfont \liminf_{m \to \infty} \density(H_{m}) = 0 \, . 
    $$
Additionally if the graph sequence $\{H_m \}_m$ is convergent in edge density, then 
    $$ \normalfont \lim_{m \to \infty} \density(H_m) = 0 \, . 
    $$
\end{lemma}
\begin{proof}
    The first part is the contra-positive of Theorem \ref{thm:main}(2). We prove it from first principles for the sake of completeness. Let us restate the square-degree property and consider its negation. If a graph sequence $\{G_n\}_n$ satisfies the square-degree property, then there exists constants $c_1 \in (0,1)$ and $N_0 \in \mathbb{N}$ such that for all $n \geq N_0$ we have
    $$ \sum \deg v_{i,n}^2 \geq c_1 \left( \sum \deg v_{i,n}\right)^2 \, . 
    $$
    The negation of square-degree property, $\neg Sq$ says that for all $c_1 > 0$ and $N_0 \in \mathbb{N}$ there exists $n \geq N_0$ such that 
    $$ \sum \deg v_{i,n}^2 < c_1 \left( \sum \deg v_{i,n}\right)^2 \, . 
    $$
    For every $N_0 \in \mathbb{N}$ there exists $n \geq N_0$ such that this inequality is satisfied. Consider
    $$ A_{c_1} = \left \{n \in \mathbb{N}: n \geq N_0, \sum \deg v_{i,n}^2 < c_1 \left( \sum \deg v_{i,n}\right)^2 \right\} \, . 
    $$
    If $|A_{c_1}|$ was finite, then we can pick $N_\nu = \max (A_{c_1}) + 1$ and for $n \geq N_{\nu}$ the inequality $\sum \deg v_{i,n}^2 < c_1 \left( \sum \deg v_{i,n}\right)^2$ would not be satisfied. Thus, the set $ A_{c_1} $ has infinitely many elements. Therefore for every $c_1 \in (0,1)$ and $N_0 \in \mathbb{N}$ there is an infinite sequence $A_{c_1}$ such that for any $n \in A_{c_1}$
    $$ \sum \deg v_{i,n}^2 < c_1 \left( \sum \deg v_{i,n}\right)^2 \, . 
    $$
    Hence we can consider a sequence of sequences $\{ A_{c_{1_i}} \}_{c_{1_i}}$ where $c_{1_i} > c_{1_j}$ when $i < j$. From this sequence set we can choose a diagonal subsequence $\{n_1, n_2, \ldots \}$ such that $n_1 \in A_{c_{1_1}}$ and $n_2 \in A_{c_{1_2}}$ and so on, such that this sequence converges to zero. 
    From equation~\eqref{eq:linedensity} recall that
    $$ \density(H_m)  = \frac{\frac{1}{2}\sum_i (\deg\,  v_{i,n})^2 - m}{\frac{1}{2}m(m-1)} \, . 
    $$
    For the diagonal subsequence selected above
    $$ \frac{ \sum \deg v_{i,n}^2}{\left( \sum \deg v_{i,n}\right)^2} = \frac{ \sum \deg v_{i,n}^2}{4m^2} \to 0 \, , 
    $$
    giving us
    $$ \liminf_{m \to \infty}\density(H_{m}) = 0 \, . 
    $$
    If $\{H_m\}_m$ is convergent, then all subsequences converge to the same limit and we get
    $$ \lim_{m \to \infty}\density(H_{m}) = 0 \, . 
    $$
\end{proof}

Lemma \ref{lemma:notsqu} coupled with Theorem \ref{thm:main} show that  dense $ \{H_m\}_m$ can only occur as a result of $\{G_n \}_n \in S_q$. This is shown in Figure \ref{fig:venndensesparsewithSq} with the shaded area representing dense $ \{H_m\}_m$.

\begin{figure}[!ht]
    \centering
    \begin{tikzpicture}
        \draw (0,0) rectangle (8cm, 2.25cm);
        \draw (2,1.125) ellipse (1.5cm and 0.75cm);
        \draw (6,1.125) ellipse (1.5cm and 0.75cm);
        \draw[fill=gray!30] (6,1.125) ellipse (0.75cm and 0.25cm);
        % \draw (1.5,1.75) circle [radius=1];
        % \draw (4.5,1.75) circle [radius=1];
        \node[] at (2,1.125) {Dense $D$};
        \node[] at (6,1.85) {Sparse $S$};
        \node[] at (6,1.125) {$S_q$};
        \node[] at (0.5,2.5) {$\{G_n\}_n$};
        \node[] at (6.8,1.5) { $\{P_n\}_n$};
        \node[] at (5.5,0.7) { $\{C_n\}_n$};
    \end{tikzpicture}
    \caption{The Euler diagram of the space of dense and sparse graph sequences, and indicate
    where there are graph sequences satisfying the square-degree property. The set $S\backslash S_q$ is non-empty as paths $\{P_n\}_n$, cycles $\{C_n\}_n$ and other graphs live here. The line graphs $\{H_m\}_m$ are dense in the shaded set  $S_q$. }
    \label{fig:venndensesparsewithSq}
\end{figure}

There are graph sequences $\{G_n\}_n \in D$ that do not converge. Similarly, there are graph sequences   $\{G_n\}_n \in S_q$ that have line graphs $\{H_m\}_m $ that do not converge. We give some examples below. 

\paragraph{Examples of non-converging graph sequences in $G$ or $H$ spaces}\ \\
Consider the following sequences $\{G_n\}_n$ and their line graph sequences $\{H_m\}_m$ where $H_m = L(G_n)$.
\begin{enumerate}
    \item $\{G_n\}_n \in D$: let $G_n \sim G(n, p_1)$ with $p_1 = 0.1$ for odd $n$ and $G_n \sim G(n, p_2)$ with $p_2 = 0.5$ for even $n$. The graph sequence $\{G_n\}_n \in D$, but this sequence does not converge; rather it has two subsequences that converge to different graphons. In contrast, the line graphs $H_m = L(G_n)$ converge to $U = 0$. 
    \item $ \{G_n\}_n \in \widebar{D \cup S}$: let $G_n \sim G(n, p)$ with $p >0$ for odd $n$ and $G_n = K_{1, n-1}$, a star graph, for even $n$. This graph sequence is neither dense nor sparse. It is not dense because $\liminf_{n \to \infty} \density(G_n) = 0$ and it is not sparse because $\limsup_{n \to \infty} \density(G_n) = p$ making $\lim_{n \to \infty} \density(G_n)$ not exist.  Thus, this sequence lives in $ \widebar{D \cup S}$.  The graph sequence  $\{G_n\}_n$ does not converge and neither does the line graph sequence $\{H_m\}_m$. 
    
    \item $\{G_n\}_n \in S \backslash S_q$: let $G_n = C_n$, a cycle graph, for odd $n$ and $G_n = K_{1,n}$, a star graph, for even $n$. The edge density of this sequence converges to 0 and $G_n \to W = 0$. This sequence is sparse. However, their line graphs do not converge. The line graph $H_m = L(G_n)$ is a cycle for odd $n$ and $H_m$ is a complete graph for even $n$. While $\{G_n\}_n$ converges to zero $\{H_m\}_m$ doesn't converge. 
    
    \item $\{G_n\}_n \in S_q$: let $G_n = K_{1,n}$ for odd $n$ and  $G_n = K_{1,n} \cup K_{1,n}$ for even $n$. These graphs are sparse and thus $G_n \to W = 0$. Moreover, these graphs are square-degree.  However, their line graphs $H_m$ do not converge. The sequence $\{H_m\}_m$ contains two subsequences that converge to two different graphons. 
\end{enumerate}

% \begin{remark} A graph sequence in $D$ does not need to converge in edge density. Consider a $G(n,p)$ graph sequence $\{G_n\}_n$  with $p = 0.1$ for odd $n$ and $p = 0.5$ for even $n$. This sequence does not converge in edge density.  Similarly, consider a sequence   $\{G_n\}_n$ of $G(n, p)$ graphs with $p= 0.1$ for odd $n$ and star graph $G_n = K_{1, n}$  for even $n$. This sequence has $\limsup_{n \to \infty} \density(G_n) = 0.1$ and $\liminf_{n \to \infty} \density(G_n) = 0$ and lives in $\widebar{D \cup S}$. 
% \end{remark}

\section{Sparse graphs with converging line graph sequences}
\label{sec:converging-sparse-graphs}

In this section we explore graph sequences converging in homomorphism density. 
The subtlety is that we define graph sequences to be dense or sparse based on their edge densities (Section~\ref{sec:notationgeneral}), but graphons are defined in terms of convergence of homomorphism density (Definition~\ref{def:convergent}).
% We suppose $\{G_n\}_n$ converges to $W$ and $\{H_m\}_m$ converges to $U$. For converging $H_m$ the square-degree property is equivalent to the max-degree condition (Lemma \ref{lemma:squareandmaxequivlence}). 
% We show that under the max-degree condition $\{H_m\}_m$ converges to a non-zero $U$. We will start with homomorphism densities. 
% \cheng{Section 2.4.1 can be copied here, and then we say that from now onwards we look at max-degree.}

% \subsubsection{Edge vs homomorphism density}\label{sec:edgevshomdensity}
% \cheng{Perhaps we could move this orphan subsubsection to Section 4?}
\paragraph{Edge vs homomorphism density}\ \\
In this study we mention different types of convergence: convergence with respect to homomorphism density (Definition~\ref{def:graphhomo}),  cut metric (Definition~\ref{def:cut2}), and edge density (Equation~\ref{eq:linedensity}). 
Homomorphism density convergence is subgraph convergence. Suppose $ \{ G_n \}_n$ converges in homomorphism density, then  for any graph $F$ the sequence $\{t(F, G_n)\}_n$ converges. 
%That is, the edge density, triangle density, 4-cycle density and all such densities converge. 
Convergence in homomorphism density is equivalent to convergence in the cut metric as shown by \cite{borgs2011limits}. In contrast, edge density convergence is the same as convergence of the single sequence $\{t(\edge, G_n)\}_n$. As edge density is given by $2|E(G_n)|/n(n-1)$ and $t(\edge, G_n) = 2|E(G_n)|/n^2$
convergence in one implies convergence in the other.
%these two quantities are close for large $n$. 
%The denominators are different because the edge density excludes the diagonal of the adjacency matrix whereas $\{t(\edge, G_n)\}_n$ includes it (see Definition~\ref{def:graphhomo}). 
However, edge density is much weaker and does not give us subgraph convergence. 

We use edge density to characterize a bigger space of graph sequences -- sequences that do not converge either in the cut metric or in edge density. The use of $\liminf$ in the definition of dense graph sequences (Definition \ref{def:dense}) means that we do not need convergence of  edge densities to call a graph sequence dense.

For the remainder of the paper, we suppose $\{G_n\}_n$ converges to $W$ and $\{H_m\}_m$ converges to $U$. For converging $H_m$ the square-degree property is equivalent to the max-degree condition (Lemma \ref{lemma:squareandmaxequivlence}). 
We show that under the max-degree condition $\{H_m\}_m$ converges to a non-zero $U$. We will start with homomorphism densities. 
%\cheng{Section 2.4.1 can be copied here, and then we say that from now onwards we look at max-degree.}

%\subsection{Conditions for non-zero graphons of line graphs} \label{sec:havingconvergence} % homomorphism density convergence
% In this section we explore graph sequences converging in homomorphism density. 
% %\cheng{I'm confused about the section title, whether we have homomorphism density convergence or edge density convergence}. \SK{My bad. Making up words. It is homomorphism density convergence.}
% We suppose $\{G_n\}_n$ converges to $W$ and $\{H_m\}_m$ converges to $U$ and show that under the square-degree property, $\{H_m\}_m$ converges to a non-zero $U$. We will start with homomorphism densities. 

\subsection{Revisiting graph homomorphisms}
% \SK{
% Need to sort out here!!! \\
% First need to show
% $$  t(F, H_m)  =  \sum_{i, j \in V(F)} \prod_{ij \in E(F) } U_m(q_i, q_j) \cdot \frac{1}{m^{|V(F)|}} 
% $$
% What we have is
% $$  t(F, W) = \int_{[0,1]^{|V(F)|}} \prod_{ij \in E(F)} W(x_i, x_j) \,   dx \, .
% $$
% Put 2.2 and 2.3 equations in Convergent sequences together to get this.
% }
% \cheng{Where do we use $t(F, H_m)$?}\SK{We don't use it. I thought it would be nice to have it. }

Recall when defining the empirical graphon (Definition~\ref{def:empiricalgraphon}) we divide the interval [0,1] into $n$ equal subintervals $I_1, I_2, \ldots, I_n$ where each $I_j$ has length $1/n$.  We use this construction in the next lemma. %Furthermore, recall that the homomorphism density $t( \edge, G_n) = 2m/n^2$ while the edge density,  $\density(G_n) = 2m/(n(n-1))$ making the two densities converge to the same limit. 

%\cheng{Are $r_i$ and $q_i$ the center points of intervals? I somehow expect intervals to have two ends.} \SK{Added a small paragraph above.}
 \begin{lemma}\label{lemma:homedgedensity}
 Let $H_m = L(G_n)$ and let $W_{n}$ be the empirical graphon of $G_n$ with $[0,1]$ divided into $n$ equal intervals $\{r_1, \ldots r_n \}$. Let $U_m$ be the empirical graphon of $H_m$ with $[0,1]$ equally divided into $m$ intervals $\{q_1, \ldots, q_m\}$.  Then $t( \edge, H_m)$ can be written as
 %\cheng{Do we want to call $t( \edge, H_m)$ ``edge density'' instead of homomorphism density? I thought homomorphism density $t( F, H_m)$ is with respect to all graphs $F$}\SK{Paragraph above addresses this. Removed homomorphism density in previous sentence.}
 $$ t(\edge, H_m) = \sum_{i,j} U_m(q_i, q_j) \cdot \frac{1}{m^2} = \sum_{\substack{i,j,k \\ i \neq j}} W_n(r_i, r_k) W_n(r_k, r_j) \cdot \frac{1}{m^2} \, . 
 $$
\end{lemma}
\begin{proof}
From Definition \ref{def:graphhomo}  
\begin{equation}
     t(F, H_m)   = \sum_{\phi: V(F) \to V(H_m)} \prod_{ij \in E(F)}  \beta_{\phi(i)\phi(j)}(H_m)  \cdot \frac{1}{m^{|V(F)|}} \, ,  \\
\end{equation}    
where $\phi$ is a mapping from $V(F)$ to $V(H_m)$ and $\beta_{ij}(H_m)$ denotes the weight of edge $ij$ in graph $H_m$, which is either 1 or 0. Thus, 
\begin{align}
    t(\edge, H_m)   & = \sum_{\phi: V(\edge) \to V(H_m)} \prod_{ij \in E(\edge)}  \beta_{\phi(i)\phi(j)}(H_m)  \cdot \frac{1}{m^{2}} \, ,  \\
    & = \sum_{ \substack{ \phi: V(\edge) \to V(H_m) \\ ij \in E(\edge) }}  \beta_{\phi(i)\phi(j)}(H_m)  \cdot \frac{1}{m^{2}} \, ,
\end{align}
where we have dropped the product term as there is only one edge. We can replace the edge weight $\beta_{\phi(i)\phi(j)}(H_m)$ with the associated value in the empirical graphon $U_m(q_{\phi(i)}, q_{\phi(j)})$  giving us
\begin{align}
    t(\edge, H_m) & = \sum_{ \substack{ \phi: V(\edge) \to V(H_m) \\ ij \in E(\edge) }}   U_m(q_{\phi(i)}, q_{\phi(j)}) \cdot \frac{1}{m^2} \, ,  \\ 
    & = \sum_{i,j} U_m(q_i, q_j) \cdot \frac{1}{m^2} \, , 
\end{align}
as $\phi$ can map the edge to any two vertices in $H_m$. Every edge in $G_n$ is mapped to a vertex in $H_m$ and 2 vertices in $H_m$ are connected if the corresponding edges in $G_n$ have a common vertex. That is, $L(\twoedges) = \edge $,  and for every edge in $H_m$ there is a corresponding set of two edges with a common vertex (\twoedges) in $G_n$. As a result the empirical graphon (Definition \ref{def:empiricalgraphon}), 
$$ U_m(q_i, q_j) = 1 \quad \text{if and only if} \quad W_n(r_k, r_\ell)W_n(r_\ell, r_s) = 1
$$
for some $k, \ell, s \in  \{1, \ldots, n \}$ with $k \neq s$.  The reason $ k \neq s$ is because we need 2 distinct edges in $G_n$ with a common vertex to make an edge in $H_m$. As a result of this one-to-one and onto mapping we have
\begin{equation}\label{eq:lemmahomdense1}
    \sum_{i, j} U_m(q_i, q_j) = \sum_{ \substack{ k, \ell, s \\ k \neq s }} W_n(r_k, r_\ell)W_n(r_\ell, r_s)  \,  
\end{equation}
giving us the desired result. 
\end{proof}

\subsection{More on converging graph sequences}

% \begin{figure}[ht]
% \centering
% \begin{tikzpicture}
%     % draw the sets
%     \draw (-3.5,0) ellipse (2.5cm and 1.5cm);
%     \draw (3.5,0) ellipse (2.5cm and 1.5cm);
%    %\filldraw[fill=red!20, draw=red!60] (3.5,0) ellipse (2.5cm and 1.5cm);

%     % the texts
%     \node at (-3.5,1.8) {$\{G_n\}_n$};
%     \node at (3.5,1.8) {$\{H_m\}_m$};

%     % the points in the sets (here I just create nodes to use them later on to position
%     % the circles and the arrows
%     \node (x1) at (-3.5,0.8) {$ \{G_n\}_n \in D \equiv W \neq 0  $};
%     \node (x2) at (-3.5,0) {$\{G_n\}_n \in S \backslash S_q \Rightarrow W = 0$};
%     \node (x3) at (-3.5,-0.8) {$\{G_n\}_n \in S_q \Rightarrow W = 0$};
%     \node (y1) at (3.5,0.5) {$\{H_m\}_m \in S \equiv U = 0$};
%     \node (y2) at (3.5,-0.5) {$\{H_m\}_m \in D \equiv U \neq 0$};

%     % draw the arrows
%     \draw[->] (x1) -- (y1);
%     \draw[->] (x2) -- (y1);
%     \draw[->] (x3) -- (y2);
%     % \draw[->] (x4) -- (y3);

% \end{tikzpicture}
% \caption{The map from converging $\{G_n\}_n$ to converging $\{H_m\}_m$, summarising Lemmas \ref{lemma:densegraphs}, \ref{lemma:converginginSq} and \ref{lemma:convergingSnotSq}. }
% \label{fig:mappingWtoU}
% \end{figure}

\begin{lemma}\label{lemma:densegraphs}
    Let $\{G_n\}_n$ be a dense graph sequence converging to $W$ and let $H_m = L(G_n)$. Then $\{ H_m \}_m$  converges to $U(x,y) = 0$ almost everywhere. 
\end{lemma}
\begin{proof}
    As $\{G_n\}_n$ is a dense graph sequence converging to $W$ 
\begin{equation}\label{eq:lemmadense1}
    \lim_{n \to \infty} t( \edge \, , G_n) = \lim_{n \to \infty} \frac{2m}{n^2}=  c > 0 \, .
\end{equation}
We will use this  limit later. Let $W_{n}$ be the empirical graphon of $G_n$ with $[0,1]$ divided into $n$ equal intervals $\{r_1, \ldots r_n \}$  and let $U_m$ be the empirical graphon of $H_m$ with $[0,1]$ equally divided into $m$ intervals $\{q_1, \ldots, q_m\}$. The homomorphism density $ \{t( \twoedges, G_n) \}_n$ is a converging sequence as $\{G_n\}_n$  converges to $W$. We have
$$  t( \twoedges, G_n)  =  \sum_{\substack{ i, j, k   }}  W(r_i, r_k)W(r_k, r_j)\cdot \frac{1}{n^{3}} \, , 
$$ % \\ i \neq j
converging as $n$ goes to infinity.
From Lemma \ref{lemma:homedgedensity} we know 
% $$  t( F, H_m) =  \sum_{i, j \in V(F)} \prod_{\substack{ ij \in E(F) \\ i \neq j}} \sum_{k \in [n]} W(r_i, r_k)W(r_k, r_j)\cdot \frac{1}{m^{|V(F)|}} \, .
% $$
\begin{align}
     t( \edge, H_m) & =   \sum_{\substack{ i, j  }}  U(q_i, q_j) \cdot \frac{1}{m^{2}} \, , \\ 
      & =  \sum_{\substack{ i, j, k \\ i \neq j  }}  W(r_i, r_k)W(r_k, r_j)\cdot \frac{1}{m^{2}} \, ,  \\
     & \leq \sum_{\substack{ i, j, k  }}  W(r_i, r_k)W(r_k, r_j)\cdot \frac{1}{m^{2}} \, , \\
     & = \sum_{\substack{ i, j, k  }} \frac{1}{n^4} W(r_i, r_k)W(r_k, r_j)\cdot \frac{1}{ \left(\frac{m}{n^2} \right)^2 } \, .
\end{align}
As $n$ and $m$ go to infinity we get
$$  \limsup_{m \to \infty}  t( \edge, H_m) = \lim_{\substack{ n \to \infty \\ m \to \infty} } \frac{1}{n} \cdot  t( \twoedges, G_n)  \cdot \frac{1}{ \left(\frac{m}{n^2} \right)^2 }  = 0 \, ,
$$
as $m/n^2$ goes to $c/2 > 0$ (equation \eqref{eq:lemmadense1}) and $t( \twoedges, G_n) $ converges. As $ t( \edge, H_m) $ lies between $0$ and $1$ we get
$$  \lim_{m \to \infty}  t( \edge, H_m) = 0 \, . 
$$
As $t( \edge, H_m) = \frac{2m'}{n'^2}$ goes to 0, where $m'$ and $n'$ denote the number of edges and vertices in $H_m$, the cut-norm (Definition \ref{def:cut1}) satisfies
$$ \lVert U_{H_m} - U \rVert_\square = \lVert U_{H_m} \rVert_\square  = \frac{2m'}{n'(n'-1)}  \rightarrow 0 \, , 
$$
where $U(x,y) = 0$.   As the cut-metric (Definition \ref{def:cut2})
    $$ \delta_\square \left(U_{H_m}, U \right) = \inf_{\varphi} \lVert U_{H_m} - U^{\varphi} \rVert \, , 
    $$
   $\{H_m\}_m$ converges to $U$ in the cut-metric as the infimum is considered and as $U^{\varphi} = U$ for $U = 0$. By Theorem~\ref{thm:BCconvergent2}, $\{ H_m \}_m$ converges to $U = 0$.
\end{proof}

%\subsubsection{Converging sparse graph sequences}
%Recall the definition of the cut-norm (Definition~\ref{def:cut1}). 
% The following lemma shows that 
% for a graph sequence $\{G_n\}_n$ satisfying the square-degree property, 
% if the sequence of line graphs $\{H_m\}_m$ converge to $U$,
% then $U$ has a strictly positive cut-norm.

The following lemma shows that 
for a graph sequence $\{G_n\}_n$ satisfying the max-degree condition, 
the sequence of line graphs $\{H_m\}_m$ converge to a graphon $U$ with a strictly positive cut-norm. Note that as $\{G_n\}_n \in S_x$, $H_m$ converges to $U$.  We show that $\lVert U \rVert_\square  > 0$. 
% where $U$ has a strictly positive cut-norm.

%\cheng{Notation question: $\lVert U(x,y) \rVert_\square$ or $\lVert U \rVert_\square$?}\SK{Sorted.}
% But Lemma~\ref{lemma:convergingSnotSq} shows that for sparse graphs that do \emph{not} have the
% square-degree property, the graphon corresponding to the line graph is uniformly zero.

\begin{lemma}\label{lemma:converginginSq}
      % Let $\{G_n\}_n \in S_q$ and let $H_m = L(G_n)$. If $\{H_m\}_m$ converges to $U$ then 
      % $U$ has strictly positive cut-norm, that is $\lVert U \rVert_\square > 0$. 
      Let $\{G_n\}_n \in S_x$ and let $H_m = L(G_n)$ with $H_m \to U$. Then 
      $U$ has strictly positive cut-norm, that is $\lVert U \rVert_\square > 0$. 
\end{lemma}
\begin{proof}
 Note that $S_x \subset S_q$. From Theorem \ref{thm:main} we know $\{G_n\}_n\in S_q \equiv \{H_m\}_m \in D$. Additionally, if $\{H_m\}_m$ converges to $U$ then $\{ t(\edge, H_m)\}_m$  converges to $t(\edge, U)$. As $t(\edge, H_m) = \frac{2m'}{n'^2}$ where $m'$ and $n'$ denote the number of edges and nodes in $H_m$ where $n' = m$, the sequence $\frac{2m'}{n'^2}$ converges to some constant $c$. But as $\{H_m\}_m \in D$ 
 $$ t(\edge, U) =  \lim_{n' \to \infty} \frac{2m'}{n'^2} = c > 0 \, ,
 $$
that is, the edge density of $\{H_m\}_m$ converges to a positive constant. The homomorphism density $t(\edge, U)$ (Definition \ref{def:graphhomo}) is given by
 $$ t(\edge, U) = \int_{[0,1]^2} U(x, y) \, dx dy   \, , 
 $$
which is equal to the cut-norm of $U$
$$  \lVert U \rVert_\square = \sup_{S ,T} \left \vert \int_{S \times T} U(x,y) dxdy \right\vert \, , 
$$
because $U(x,y) \in [0,1]$ and the supremum is achieved when $S = T = [0,1]$, giving us 
$$ \lVert U \rVert_\square  =  t(\edge, U)  > 0 \, . 
$$
 % Let $\{G_n\}_n\in S$ be a sparse graph sequence. Let $\{H_m\}_m$ be the corresponding sequence of line graphs with $H_m = L(G_n)$. Then $\{G_n\}_n\in S_q \equiv \{H_m\}_m \in D$, i.e., $\{G_n\}_n$ satisfies $Sq$ if and only if $\{H_m\}_m$ is dense. 
\end{proof}

Lemma~\ref{lemma:convergingSnotSq} shows that sparse graphs that do \emph{not} satisfy the
max-degree property have line graphs converging to the zero graphon. %,  the graphon corresponding to the line graph is uniformly zero.

\begin{lemma}\label{lemma:convergingSnotSq}
    Let $\{G_n\}_n \in S\backslash S_x$ and let $H_m = L(G_n)$ with $H_m \to U$. Then $U = 0$ almost everywhere. 
\end{lemma}
\begin{proof}
    As $\{G_n\}_n \in S$, it is sparse and it converges to $W = 0$. If $\{G_n\}_n \in S_x$
    \[
    \lim_{m \to \infty}\density(H_{m}) = c > 0 \, .  
    \]
    Thus, for converging $\{H_{m}\}_m \notin S_x$ we have
    \[
    \lim_{m \to \infty}\density(H_{m}) = 0 \, .
    \]
    % From Lemma \ref{lemma:notsqu}\,  if  $\{G_n \}_n \notin S_q$ 
    %   $$ \liminf_{m \to \infty}\density(H_{m}) = 0 \, ,  
    % $$
    % which holds for $\{G_n \}_n \notin S_x$  as $S_x \subset S_q$.  % this last part is nto true
    % As $\{H_m\}_m$ converges to $U$, the edge densities converge and we get $\lim_{m \to \infty} \density(H_m) = 0$. 
    The empirical graphon $U_{H_m}$ converges to $U$ and we have the cut norm (Definition \ref{def:cut1}) of the empirical graphon
    $$ \lVert U_{H_m} \rVert  = \frac{2m'}{n'(n'-1)} \to 0
    $$
     where $m'$ and $n'$ denote the number of edges and nodes in $H_m$ giving us 
    $$ \lim_{m \to \infty}   \lVert U_{H_m} - U \rVert = 0 \, , 
    $$
    where $U = 0$. %As the cut-metric (Definition \ref{def:cut2})
    % $$ \delta_\square \left(U_{H_m}, U \right) = \inf_{\varphi} \lVert U_{H_m} - U^{\varphi} \rVert \, , 
    % $$
    % we get the result. 
\end{proof}

For the set of converging graph sequences $\{G_n\}_n$ the complement of dense graph sequences are sparse, i.e.,  $\widebar{D} = S$. If the line graphs $\{H_m\}_m$ converge as well, then the square-degree and max-degree property are equivalent (Lemma \ref{lemma:squareandmaxequivlence}), i.e., $S_q = S_x$. Figure \ref{fig:vennconverging} shows the updated Euler diagram. 
% For graph sequences $\{G_n\}_n$ converging in homomorphism density the Euler diagram of sparse and dense graphs is given in Figure \ref{fig:vennconverging}. 

\begin{figure}[!ht]
    \centering
    \begin{tikzpicture}
        \draw (0,0) rectangle (5cm, 2.25cm);
        \draw (3,1.125) ellipse (1.5cm and 0.75cm);
        \draw[fill=gray!30] (3,1.125) ellipse (0.75cm and 0.25cm);
        % \draw (1.5,1.75) circle [radius=1];
        % \draw (4.5,1.75) circle [radius=1];
        \node[] at (3,1.85) {Sparse};
        \node[] at (3,1.125) {$S_x$};
        \node[] at (2.5,2.5) {Converging $\{G_n\}_n$ and $\{H_m\}_m$};
        \node[] at (1,0.5) {Dense};
    \end{tikzpicture}
    \caption{The Euler diagram in Figure \ref{fig:venndensesparsewithSq} updated for converging $\{G_n\}_n$ and $\{H_m\}_m$.  }
    \label{fig:vennconverging}
\end{figure}

Lemmas \ref{lemma:densegraphs}, \ref{lemma:converginginSq} and \ref{lemma:convergingSnotSq} can be used map different instances of $W$ to $U$ depending on the characteristics of $\{G_n\}_n$. For $W$ and $U$ to exist both sequences $\{G_n\}_n$ and $\{H_m\}_m$ need to converge. Figure \ref{fig:mappingWtoU} shows this relationship.
% converging graph sequences $\{G_n\}_n$ and their limit $W$ to converging $\{H_m\}_m$ and their limit $U$.

\begin{figure}[ht]
\centering
\begin{tikzpicture}
    % draw the sets
    \draw (-3.5,0) ellipse (2.5cm and 1.5cm);
    \draw (3.5,0) ellipse (2.5cm and 1.5cm);
   %\filldraw[fill=red!20, draw=red!60] (3.5,0) ellipse (2.5cm and 1.5cm);

    % the texts
    \node at (-3.5,1.8) {$\{G_n\}_n$};
    \node at (3.5,1.8) {$\{H_m\}_m$};

    % the points in the sets (here I just create nodes to use them later on to position
    % the circles and the arrows
    \node (x1) at (-3.5,0.8) {$ \{G_n\}_n \in D \equiv W \neq 0  $};
    \node (x2) at (-3.5,0) {$\{G_n\}_n \in S \backslash S_x \Rightarrow W = 0$};
    \node (x3) at (-3.5,-0.8) {$\{G_n\}_n \in S_x \Rightarrow W = 0$};
    \node (y1) at (3.5,0.5) {$\{H_m\}_m \in S \equiv U = 0$};
    \node (y2) at (3.5,-0.5) {$\{H_m\}_m \in D \equiv U \neq 0$};

    % draw the arrows
    \draw[->] (x1) -- (y1);
    \draw[->] (x2) -- (y1);
    \draw[->] (x3) -- (y2);
    % \draw[->] (x4) -- (y3);

\end{tikzpicture}
\caption{The map from converging $\{G_n\}_n$ to converging $\{H_m\}_m$, summarising Lemmas \ref{lemma:densegraphs}, \ref{lemma:converginginSq} and \ref{lemma:convergingSnotSq}. }
\label{fig:mappingWtoU}
\end{figure}

\subsection{Orthogonal spaces}
% \begin{lemma}\label{lemma:WandUold}
% Suppose $\{ G_n\}_n$ converges to $W$ and  $\{ H_m\}_m$ converges to $U$ where $H_m = L(G_n)$. Then the inner product 
% $$ \langle W, U \rangle = \int_{[0,1]^2} W(x,y) U(x, y) \, dx dy = 0 \, . 
% $$
%  Thus, graphons $U$ obtained from line graphs are orthogonal to graphons $W$ with respect to the above inner product.   

 \begin{lemma}\label{lemma:WandU}
Suppose $\{ G_n\}_n$ converges to $W$ and  $\{ H_m\}_m$ converges to $U$ where $H_m = L(G_n)$. Let $W_{G_n}$ and $U_{H_m}$ denote the empirical graphons (Definition \ref{def:empiricalgraphon}) of $G_n$ and $H_m$ respectively. Then the inner product defined by
 \[
 \langle W_{G_n}, U_{H_m} \rangle = \int_{[0,1]^2} W_{G_n}(x,y) U_{H_m}(x, y) \, dx dy
 \]
 satisfies
 \[
 \lim_{n, m \to \infty}  \langle W_{G_n}, U_{H_m} \rangle  = 0 \, . 
 \]
\end{lemma}
\begin{proof}
The limit
\begin{align}
     \lim_{n, m \to \infty}  \langle W_{G_n}, U_{H_m} \rangle  & = \lim_{n, m \to \infty} \int_{[0,1]^2} W_{G_n}(x,y) U_{H_m}(x, y) \, dx dy \, , \\
      & =  \int_{[0,1]^2} \lim_{n, m \to \infty} \left( W_{G_n}(x,y) U_{H_m}(x, y) \right) \, dx dy \, , 
\end{align}
as both $W_{G_n}(x,y)$ and $U_{H_m}(x, y)$ are bounded. As $W_{G_n} \to W$, $U_{H_m} \to U$ and for  converging sequences $\{ G_n\}_n$ and  $\{ H_m\}_m$ we have $W = 0$ or $U = 0$ ((Lemmas \ref{lemma:densegraphs}, \ref{lemma:converginginSq} and \ref{lemma:convergingSnotSq})) we get the result. 
   %  The graphon $W \neq 0$ only when $\{G_n\}_n \in D$. When $\{G_n\}_n \in D$ we have $\{H_m\}_m \in S$ giving $U = 0$. The graphon $U \neq 0 $ only when $\{G_n\}_n \in S_q$ implying $W = 0$ as $S_q \subset S$. %Figure \ref{fig:mappingWtoU} explains this Lemma. 
\end{proof}
 %For  converging sequences $\{ G_n\}_n$ and  $\{ H_m\}_m$ we have $W = 0$ or $U = 0$ (Lemmas \ref{lemma:densegraphs}, \ref{lemma:converginginSq} and \ref{lemma:convergingSnotSq}).

%For example, we can mix 2 sequences, say a sequence of star graphs and a sequence of paths, and the mixed sequence would still converge to $W = 0$.  

%The fact that $W=0$ for sparse graphs is implied by the regularly expressed consequence of  Aldous Hoover theorem -- graphs represented by exchangeable random arrays (graphons) are either empty or dense. 

\section{Dense graphs $\{G_n\}_n \in D$}\label{sec:dense}

% \cheng{Style question: I generally try to avoid orphan sections (Section 4.1 without 4.2), (Section 5.1 without 5.2) and (Section 7.1 without 7.2). We could introduce 5.1 for complete graphs, and 7.1 for stars? Alternatively we remove the subsection commands, and just use paragraphs or something. And Section 4 can have three subsections.}

In this section, we consider two examples of dense graphs: complete graphs and  Erd\H{o}s–R\'{e}nyi $G(n,p)$ graphs. The Erd\H{o}s–R\'{e}nyi model results in a dense non-zero graphon $W$, 
and we show in Section~\ref{sec:erdos-renyi} that the edge density
of the line graph $H_m$ converges to zero with an exponential rate resulting in a zero graphon for line graphs.

% two models of graphs that are popular in social network analysis,
% namely preferential attachment models \citep{albert2002statistical} and Erd\H{o}s–R\'{e}nyi model $G(n,p)$. As noted in Figure~\ref{fig:resultsdiagram}, the two models have contrasting behaviours.
% We show in Section~\ref{sec:prefattachment} that superlinear preferential attachment models
% satisfy the square-degree property almost surely, and the corresponding line graphs $H_m$
% converges to a non-zero graphon $U$.
% The Erd\H{o}s–R\'{e}nyi model results in a dense non-zero graphon $W$, 
% and we show in Section~\ref{sec:erdos-renyi} that the edge density
% of $H_m$ convergence to zero with an exponential rate.
% \cheng{an updated version of this could fit at the beginning of 5.1}

\subsection{Complete graphs}\label{sec:completegraphs}

\begin{lemma}\label{thm:proplinegraphs1}
    Let  $G_n$ be the complete graph $K_n$ and $H_m = L(G_n)$. Then the graph sequence $\{G_n\}_n$ with $n \to \infty$ satisfies
        \[
    \lim_{m \to \infty} \density(H_m) = 0 \, , \quad \{H_m\}_m \in S \backslash S_q ,  \quad  \text{and} \quad U = 0 \, . 
    \]
\end{lemma}
\begin{proof}
Recall that 
$$    \density(H_m) = \frac{2 |E(H_m)|}{|V(H_m)|(|V(H_m)| - 1)} =  \frac{\frac{1}{2}\sum (\deg\,  v^2) - m}{\frac{1}{2}m(m-1)} \, . 
$$
 As $\density(K_n) = 1$, the sequence $\{K_n\}_n$ is dense, i.e, $ \{K_n\}_n \in D$. For $K_n$,  $\deg v_i = n-1$ and $m = n(n-1)/2$ giving us
        \begin{align*}
         \density(H_m) & = \frac{\frac{1}{2} n(n-1)^2 - \frac{1}{2}n(n-1) }{ \frac{1}{2} \frac{1}{2}n(n-1)(\frac{1}{2}n(n-1) - 1) } \, ,  \\
         & = \frac{\frac{1}{2} n(n-1)(n-1 -1)}{\frac{1}{2} \frac{1}{2}n(n-1)\frac{1}{2}(n(n-1) - 2)} \\
         & =  \frac{n-2}{ \frac{1}{4} (n+1)(n-2) } \, , \\
         & = \frac{4}{n + 1} \, , \\
      \text{making} \quad   \lim_{m \to \infty}\density(H_m) & = 0  \Rightarrow \{H_m\}_m \in S.  
        \end{align*}
When $ \density(H_m) \to 0$ we have $H_m \to U = 0$, because  $\int_{0}^1\int_{0}^1 U_{H_m} dx dy = \frac{2 |E(H_m)|}{|V(H_m)|^2}$ where $U_{H_m}$ denotes the empirical graphon of $H_m$  (Definition \ref{def:empiricalgraphon}).
\end{proof}

\subsection{Erdős–Rényi graphs} \label{sec:erdosrenyigraphs}
\label{sec:erdos-renyi}
We consider Erdős–Rényi graphs $G_n \sim G(n,p)$, where $n$ denotes the number of nodes and $p$ denotes the edge probability. 
%The Erdős–Rényi model $G(n,p)$ describes graphs of $n$ vertices with  edge probability $p$. %, where $p$ is a parameter. Each edge is equally likely to be included in the graph. 
The degree distribution for any vertex in $G_n$ is binomial with parameters $n-1$ and $p$. %, $\Binom(n-1, p)$.
%The degree distribution for any vertex in a $G(n, p)$ graph is binomial with parameters $n−1$ and $p$. 
For a given $n$ and $p$ there exists a graph distribution as different edges can be included or left out in different graphs. Expectations are calculated with respect to this graph distribution. Lemmas \ref{lemma:erdospartresult1}, \ref{lemma:erdospartresult2} and \ref{lemma:erdospartresult3} are used in Theorem \ref{thm:erdosrenyi} to prove that the edge density of $H_m$ converges to zero with high probability.

%\cheng{Define $G(n,p)$ model, in particular the distribution for $n$. I assume $p$ is a parameter?}

\begin{lemma}\label{lemma:erdospartresult1}
Consider the graph $G_n $ sampled from a $G(n,p)$ model and suppose $G_n$ has $n$ nodes and $m$ edges. Let $X_{ij}$ denote the random variable corresponding to the edge between vertices $i$ and $j$, i.e., $X_{ij} = 1$ if the edge exists and $0$ otherwise. Let $Y_j = \sum_i X_{ij}$,  $\mu = \mathbb{E}[Y_j]$ and $\bar{m} = \mathbb{E}[m]$. Let $Y_{\ssq} = \sum_j Y_j^2$, and $m_{\sq} = m^2$. Then for a given $\alpha \in (0,1)$ and $c > 0$ we have
\[
P\left[ Y_{\ssq}\geq c m_{\sq} \vert m_{\sq}  \leq (1 - \alpha)^2\bar{m}^2  \right] P\left[ m_{\sq} \leq (1 - \alpha)^2\bar{m}^2  \right] \leq \exp\left( - \frac{\alpha^2 pn(n-1)}{4} \right) \, . 
\]
\end{lemma}
\begin{proof}
For a given $\alpha \in (0,1)$ we get the following Chernoff-Hoeffding bounds \citep{Frieze2015book} for $m$:
\begin{align}
      P \left[m \leq (1 - \alpha)\bar{m} \right] & \leq \exp\left( - \frac{\alpha^2 \bar{m}}{2} \right) \, ,  \label{eq:erdosrenyi6} \\
  \text{giving us } \quad        P\left[m^2 \leq (1 - \alpha)^2\bar{m}^2 \right] & \leq \exp\left( - \frac{\alpha^2 \bar{m}}{2} \right) \, ,  \label{eq:erdosrenyi7}   
\end{align} 
as $m$ is positive.  As the probability $P\left[ Y_{\ssq}\geq c m_{\sq} \vert m_{\sq}  \leq (1 - \alpha)^2\bar{m}^2  \right] \leq 1$, $m_{\sq} = m^2$ and  $\bar{m} = pn(n-1)/2$ we get the desired result. 
\end{proof}

\begin{lemma}\label{lemma:erdospartresult2}
Consider the graph $G_n $ sampled from a $G(n,p)$ model and suppose $G_n$ has $n$ nodes and $m$ edges. Let $X_{ij}$ denote the random variable corresponding to the edge between vertices $i$ and $j$, i.e., $X_{ij} = 1$ if the edge exists and $0$ otherwise. Let $Y_j = \sum_i X_{ij}$,  $\mu = \mathbb{E}[Y_j]$ and $\bar{m} = \mathbb{E}[m]$. Let $Y_{\ssq} = \sum_j Y_j^2$, and $m_{\sq} = m^2$. Then for a given $\alpha \in (0,1)$ and $c > 0$ we have
\[
P\left[ Y_{\ssq}\geq c m_{\sq} \vert m_{\sq}  \geq (1 + \alpha)^2\bar{m}^2  \right] P\left[ m_{\sq} \geq (1 + \alpha)^2\bar{m}^2  \right] \leq \exp\left( - \frac{\alpha^2 pn(n-1)}{6} \right) \, . 
\]
\end{lemma}
\begin{proof}
    The proof is similar to Lemma \ref{lemma:erdospartresult1} with the only difference being the Chernoff-Hoeffding bound, which changes to:
$$  P\left[m \geq (1 + \alpha)\bar{m} \right]  \leq \exp\left( - \frac{\alpha^2 \bar{m}}{3} \right) \, . 
$$
\end{proof}

\begin{lemma}\label{lemma:erdospartresult3}
Consider the graph $G_n $ sampled from a $G(n,p)$ model and suppose $G_n$ has $n$ nodes and $m$ edges. Let $X_{ij}$ denote the random variable corresponding to the edge between vertices $i$ and $j$, i.e., $X_{ij} = 1$ if the edge exists and $0$ otherwise. Let $Y_j = \sum_i X_{ij}$,  $\mu = \mathbb{E}[Y_j]$ and $\bar{m} = \mathbb{E}[m]$. Let $Y_{\ssq} = \sum_j Y_j^2$, and $m_{\sq} = m^2$. Then for a fixed $c >0$ and fixed $\alpha \in (0,1)$ for $n > \frac{4}{c(1-\alpha)^2}$ and $\beta = \frac{\sqrt{cn}(1 - \alpha)}{2} - 1$ we have
\begin{align}
    P\left[ Y_{\ssq}\geq c m_{\sq} \vert  (1 - \alpha)^2\bar{m}^2 \leq  m_{\sq} \leq  (1 + \alpha)^2\bar{m}^2 \right] & P\left[ (1 - \alpha)^2\bar{m}^2 \leq  m_{\sq} \leq  (1 + \alpha)^2\bar{m}^2  \right]  \\
    & \leq \exp\left(\ln n - \frac{\beta^2 p (n-1)}{3} \right) \, . 
\end{align}
\end{lemma}
\begin{proof}
    We focus on the term $ P\left[ Y_{\ssq}\geq c m_{\sq} \vert  (1 - \alpha)^2\bar{m}^2 \leq  m_{\sq} \leq  (1 + \alpha)^2\bar{m}^2 \right]$. We know that $\mu = p(n-1)$ and $\bar{m} = pn(n-1)/2$.  As $ m_{\sq} \in \left[ (1 - \alpha)^2\bar{m}^2, \,   (1 + \alpha)^2\bar{m}^2 \right]$ we get
\begin{align}
    P\left[ Y_{\ssq}\geq c m_{\sq} \vert  (1 - \alpha)^2\bar{m}^2 \leq  m_{\sq} \leq  (1 + \alpha)^2\bar{m}^2 \right] & \leq P\left[ Y_{\ssq}\geq c(1 - \alpha)^2\bar{m}^2 \right] \, , \label{eq:erdosrenyi10} \\
   & = P\left[ Y_{\ssq} \geq c\left( \frac{n\mu}{2} \right)^2 (1 - \alpha)^2  \right]  \, , \label{eq:erdosrenyi11} \\
   & =  P\left[  Y_{\ssq} \geq n \mu^2 (1 + \beta)^2\right] \, , \label{eq:erdosrenyi12} 
\end{align} 
where we have substituted  $\bar{m} = \frac{n \mu}{2}$ in equation \eqref{eq:erdosrenyi11}  and rearranged the terms for $\beta$. For $n > \frac{4}{c(1 - \alpha)^2}$, we get $\frac{\sqrt{cn}(1 - \alpha)}{2} > 1$ making $\beta > 0$. 

 For a given $\beta > 0$, we get the following Chernoff-Hoeffding bound for $Y_j$:
    \begin{align}
        P\left[Y_j \geq (1 + \beta) \mu \right] \leq \exp \left( -\frac{\beta^2 \mu}{3}\right) \, .  \label{eq:erdosrenyi1}\\
     %\cheng{Some more words about which properties enable this and the next step}} \\ 
      \text{As } Y_j \geq 0 \text{ we have }  \quad   P\left[Y_j^2 \geq (1 + \beta)^2 \mu^2 \right] \leq \exp \left( -\frac{\beta^2 \mu}{3}\right) \, ,  \label{eq:erdosrenyi2} \\
   \text{and from Boole's inequality} \quad        P\left[Y_{\ssq} \geq n(1 + \beta)^2 \mu^2 \right] \leq n \exp \left( -\frac{\beta^2 \mu}{3}\right) \, .  \label{eq:erdosrenyi3} \\        
    \end{align}
Substituting equation \eqref{eq:erdosrenyi3} in equation \eqref{eq:erdosrenyi12}  we get
\begin{align}
    P\left[ Y_{\ssq}\geq c m_{\sq} \vert  (1 - \alpha)^2\bar{m}^2 \leq  m_{\sq} \leq  (1 + \alpha)^2\bar{m}^2 \right] & \leq n \exp \left( -\frac{\beta^2 \mu}{3}\right) \, , \\
    & \leq  \exp \left( \ln n -\frac{\beta^2 \mu}{3}\right)
\end{align}
for $n > \frac{4}{c(1-\alpha)^2}$ and $\beta = \frac{\sqrt{cn}(1 - \alpha)}{2} - 1$. 
\end{proof}

\begin{theorem}\label{thm:erdosrenyi}
    Let $G_n$ be an Erdős–Rényi graph sampled from a $G(n,p)$ model and suppose $G_n$ has $n$ nodes and $m$ edges. Let $H_m = L(G_n)$. Then for any $c \in(0, 1)$, the edge density of $H_m$ satisfies 
    \[
    P\left[ \density(H_m) \geq c \right]  \leq  \exp\left( - \frac{\alpha^2pn(n-1) }{4} \right) + \exp \left( \ln n -\frac{\beta^2 p(n-1)}{3}\right) + \exp\left( - \frac{\alpha^2pn(n-1)}{6} \right) \, ,
    \]
    where $\alpha \in (0,1)$,  $n > \frac{4}{c(1-\alpha)^2}$ and $\beta = \frac{\sqrt{cn}(1 - \alpha)}{2} - 1$. 
Therefore, as $n$ and $m$ go to infinity the edge density of $H_m$ satisfies
\begin{equation}\label{eq:result2erdosrenyithm}
\normalfont     \lim_{m \to \infty} P\left[ \density(H_m) = 0 \right] = 1\, , 
\end{equation}
resulting in 
\[
 \lim_{m \to \infty} P\left[ U = 0 \right] = 1\, , 
\]
\end{theorem}
\begin{proof}
    Let $X_{ij}$ denote the Bernoulli random variable corresponding to the edge between nodes $i$ and $j$ in $G_n$ and let $X_{ij} = 1$ if the edge is present and $0$ otherwise. Let $Y_j = \sum_i X_{ij}$. Then the degree of each node $j$ in $G_n$ is given by $\deg v_j = Y_j$. Let $\mu = \mathbb{E}[Y_j]$ and $\bar{m} = \mathbb{E}[m]$. We know that $\mu = p(n-1)$ and $\bar{m}  = pn(n-1)/2$. Let $Y_{\ssq} = \sum_j Y_j^2$, and $m_{\sq} = m^2$ where ssq denotes the  sum of squares and sq denotes square.
    
We fix $c$ and $\alpha$ such that $c, \alpha \in (0,1)$ and  compute $P\left[ Y_{\ssq} \geq c m_{\sq} \right]$ using the law of total probability 
\begin{align}
    P\left[ Y_{\ssq} \geq  c m_{\sq} \right]  & = P\left[ Y_{\ssq}\geq c m_{\sq} \vert m_{\sq}  \leq (1 - \alpha)^2\bar{m}^2  \right] P\left[ m_{\sq} \leq (1 - \alpha)^2\bar{m}^2 \right]  + \\
    & \hspace{-1.5cm} P\left[ Y_{\ssq}\geq c m_{\sq} \vert  (1 - \alpha)^2\bar{m}^2 \leq  m_{\sq} \leq  (1 + \alpha)^2\bar{m}^2 \right] P\left[ (1 - \alpha)^2\bar{m}^2 \leq  m_{\sq} \leq  (1 + \alpha)^2\bar{m}^2  \right]  + \\
    & \hspace{-1.5cm} P\left[ Y_{\ssq}\geq c m_{\sq} \vert m_{\sq} \geq (1 + \alpha)^2\bar{m}^2  \right] P\left[ m_{\sq} \geq (1 +  \alpha)^2\bar{m}^2 \right]  \, , \ \label{eq:erdosrenyi8} \\ 
    & \leq \exp\left( - \frac{\alpha^2pn(n-1) }{4} \right) + \exp \left( \ln n -\frac{\beta^2 p(n-1)}{3}\right) + \exp\left( - \frac{\alpha^2pn(n-1)}{6} \right) \, ,  \label{eq:erdosrenyi9}
 \end{align}   
from Lemmas \ref{lemma:erdospartresult1}, \ref{lemma:erdospartresult2} and \ref{lemma:erdospartresult3} for $n > \frac{4}{c(1-\alpha)^2}$ and $\beta = \frac{\sqrt{cn}(1 - \alpha)}{2} - 1$.     

For a fixed $c \in (0,1) $ we get
$$ \lim_{n,m \to \infty}  P\left[ Y_{\ssq} \geq  c m_{\sq} \right]  = 0
$$ 
As 
$$ P\left[ \frac{1}{2} Y_{\ssq} \geq  \frac{1}{2} c m_{\sq} \right] = P\left[\frac{1}{2}\sum_j Y_j^2  - m \geq  \frac{1}{2} c m^2 -m  \right]
$$
For $c \in (0,1)$ we have $ \frac{1}{2}cm^2 - \frac{1}{2}cm > \frac{1}{2}cm^2 - m$ giving us
$$
P\left[\frac{1}{2}\sum_j Y_j^2  - m \geq  \frac{1}{2} c m^2 - \frac{1}{2} c m  \right] \leq P\left[\frac{1}{2}\sum_j Y_j^2  - m \geq  \frac{1}{2} c m^2 -m  \right] =  P\left[ Y_{\ssq} \geq  c m_{\sq} \right] \, . 
$$
As $n$ and $m$ go to infinity
\begin{align}
   \lim_{n, m \to \infty} P\left[ \frac{\frac{1}{2}\sum_j Y_j^2  - m }{\frac{m(m-1)}{2}}  \geq c \right] & = 0 \, .   \\
   \text{Using the line graph edge density} & \text{ in equation \eqref{eq:linedensity}} \\
   \lim_{m \to \infty} P\left[ \density(H_m) \geq c \right] & = 0 \, , 
\end{align}
giving us the first result. 
%obtaining equation \eqref{eq:result1erdosrenyithm}.
Taking the complement we have
$$  
      \lim_{m \to \infty} P\left[ \density(H_m) < c \right] = 1\, ,
$$
for a fixed $c \in (0, 1)$. As this is true for any $c \in (0, 1)$ we have
$$  \lim_{m \to \infty} P\left[ \density(H_m) = 0 \right] = 1\, . 
$$
The edge density of a sequence of line graphs converging to zero, i.e. $ \density(H_m) \to 0$ is equivalent to  $U \to 0$,
because  $\int_{0}^1\int_{0}^1 U_{H_m} dx dy = \frac{2 |E(H_m)|}{|V(H_m)|^2}$ where $U_{H_m}$ denotes the empirical graphon of $H_m$  (Definition \ref{def:empiricalgraphon}). This gives us 
\[
 \lim_{m \to \infty} P\left[ U = 0 \right] = 1\, .
\]
\end{proof}  

% ---------------------------------------------------------------------------
\section{Sparse but not square-degree $\{G_n\}_n \in S \backslash S_q$}\label{sec:SparseNotInSq}
% ---------------------------------------------------------------------------

In this section we consider some simple, deterministic graphs: $r$-regular graphs for fixed $r$, paths and cycles. While path and cycle sequences $\{G_n\}_n$ are obviously sparse, $r$-regular graphs are slightly different. For constant $r$, $r$-regular graphs are sparse, but if $r = c n$ for some constant $c \in (0,1)$, then the resulting sequence is dense. Here we consider a fixed $r$.    

\begin{theorem}\label{thm:proplinegraphs}
    Let  $G_n$ be one of the following graphs with $n$ nodes and $m$ edges.
    \begin{enumerate}
        \item $G_n$ is an $r$-regular graph.
        \item $G_n$ is a path $P_n$.
        \item $G_n$ is a cycle $C_n$.
    \end{enumerate}
    Then the graph sequence $\{G_n\}_n$ with $n \to \infty$ satisfies
        \[
    \lim_{m \to \infty} \density(H_m) = 0 \, , \quad \{H_m\}_m \in S \backslash S_q ,  \quad  \text{and} \quad U = 0 \, , 
    \]
    where $H_m = L(G_n)$. 
\end{theorem}
\begin{proof}
To prove that the edge density of line graph $H_m$ converges to zero, we use properties of line graphs (Lemma~\ref{lemma:linegraphs}).
Recall that 
$$    \density(H_m) = \frac{2 |E(H_m)|}{|V(H_m)|(|V(H_m)| - 1)} =  \frac{\frac{1}{2}\sum (\deg\,  v^2) - m}{\frac{1}{2}m(m-1)} \, . 
$$
When $ \density(H_m) \to 0$ we have $H_m \to U = 0$, because  $\int_{0}^1\int_{0}^1 U_{H_m} dx dy = \frac{2 |E(H_m)|}{|V(H_m)|^2}$ where $U_{H_m}$ denotes the empirical graphon of $H_m$  (Definition \ref{def:empiricalgraphon}). For complete graphs, $r$-regular graphs, paths and cycles we show that $ \density(H_m) \to 0$.

    \begin{enumerate}

       \item  Suppose $G_n$ is an $r$-regular graph. As $n$ grows each $G_n$ is connected to $r$ nodes.   Then $\deg v_i = r$ and $m = rn/2$ for even $n$. The ratio $m/n^2 = r/2n$ assigning $\{G_n\}_n \in S$. %As each node has the same degree in $G_n$ it does not satisfy the square-degree property, i.e. $\{G_n\}_n \in S \backslash S_q$.
       The density of $H_m$ is given by
       \begin{align}
            \density(H_m) & = \frac{ \frac{1}{2}nr^2 - \frac{1}{2}nr }{ \frac{1}{2}   \frac{1}{2} rn   ( \frac{1}{2} rn - 1 )  } \, ,  \\
            & = \frac{ \frac{1}{2}rn(r-1)}{ \frac{1}{2} \frac{1}{2} rn \frac{1}{2}(rn -2 )} \, ,  \\
            & = \frac{4(r-1)}{rn - 2} \, .  \label{eq:linedensityrregular} 
       \end{align}
     Thus, $\lim_{m \to \infty}\density(H_m) = 0$, making both $\{G_n\}_n, \{H_m\}_m \in S$.   
     Using Theorem \ref{thm:main} we can conclude $\{G_n\}_n \notin S_q$ because  $\{G_n\}_n \in S_q  \iff \{H_m\}_m \in D$. Hence $\{G_n\}_n \in S\backslash S_q$.  
     As $G$ is an $r$-regular graph, $H$ is a $2(r-1)$-regular graph with $\frac{nr}{2}$ vertices (Lemma \ref{lemma:linegraphs}-\ref{lemma:linegraphs:r-regulargraph}).  Thus, using the same reasoning we have $\{H_m\}_m \in S\backslash S_q$.   
     This is an example where both graph sequences $\{G_n\}_n$ and $\{H_m\}_m$ are sparse and both $\{G_n\}_n, \{H_m\}_m \in S\backslash S_q$.

    %  \item Suppose $G_n$ is the complete graph $K_n$. As $\density(K_n) = 1$, the sequence $\{K_n\}_n$ is dense, i.e, $ \{K_n\}_n \in D$. For $K_n$,  $\deg v_i = n-1$ and $m = n(n-1)/2$ giving us
    %     \begin{align*}
    %      \density(H_m) & = \frac{\frac{1}{2} n(n-1)^2 - \frac{1}{2}n(n-1) }{ \frac{1}{2} \frac{1}{2}n(n-1)(\frac{1}{2}n(n-1) - 1) } \, ,  \\
    %      & = \frac{\frac{1}{2} n(n-1)(n-1 -1)}{\frac{1}{2} \frac{1}{2}n(n-1)\frac{1}{2}(n(n-1) - 2)} \\
    %      & =  \frac{n-2}{ \frac{1}{4} (n+1)(n-2) } \, , \\
    %      & = \frac{4}{n + 1} \, , \\
    %   \text{making} \quad   \lim_{m \to \infty}\density(H_m) & = 0  \Rightarrow \{H_m\}_m \in S.  
    %     \end{align*}
    
    % From part (1) we get the same density for $H_m$ by substituting $r = n-1$. Similar to part (1) we know that $H_m$ is a $2(n-2)$-regular graph with $m$ vertices where $m = \frac{1}{2}n(n-1)$ and as such $\{H_m\}_m \in S\backslash S_q$.  
   
   \item Suppose $G_n$ is a path. Then $m = n-1$ and the starting and ending vertices have degree 1 and the rest have degree 2. Thus, 
    \begin{align*}
        \density(H_m) & = \frac{\frac{1}{2} (1^2 + (n - 2)2^2 + 1 ^2) - m }{\frac{1}{2}m(m-1)} \, , \\
        & = \frac{\frac{1}{2}(2 + 4(m-1)) - m}{\frac{1}{2}m(m-1)} \, ,  \\
        & =  \frac{1 + 2(m-1) - m}{\frac{1}{2}m(m-1)} \, ,  \\
        & = \frac{m-1}{\frac{1}{2}m(m-1)}  \, , \\
        & = \frac{2}{m} = \frac{2}{n-1} \, . 
    \end{align*}
    Thus, $\lim_{m \to \infty}\density(H_m)  = 0$. 
    The edge density can also be derived by recognizing a path of $n$ vertices gives rise to a line graph that is a path of $n-1$ vertices (Lemma~\ref{lemma:linegraphs}-\ref{lemma:linegraphs:paths}). 
    %In the above computation we used equation \eqref{eq:linedensity} to obtain the result. 
    Using the same reasoning as previously for $r$-regular graphs, we can conclude that both $\{G_n\}_n, \{H_m\}_m \in S\backslash S_q$.    %making $W(x,y) = 0$ for all $x$ and $y$. T

    \item Suppose $G_n$ is a cycle, i.e. $G_n = C_n$. Then $n = m$ and all vertices have degree 2. This is a 2-regular graph. Using equation \eqref{eq:linedensityrregular} we get
    \begin{align}
          \density(H_m) & = \frac{4(r-1)}{rn - 2} \, , \\
           & = \frac{4}{2n - 2}  = \frac{2}{n-1} \, , 
    \end{align}
    which limits to zero. From Lemma \ref{lemma:linegraphs}-\ref{lemma:linegraphs:cyle} we know that $L(C_n) = C_n$. 
    Here too both $\{G_n\}_n, \{H_m\}_m \in S\backslash S_q$ as previously.    % with $W = 0$. 
\end{enumerate}
\end{proof}

%-----------------------------------------------------------------
\section{Square-degree graphs $\{G_n\}_n \in S_q$} \label{sec:sq}
%-----------------------------------------------------------------
\subsection{Star graphs}\label{sec:stars}

\begin{lemma}\label{lemma:stargraphs1}
    Let  $\{G_n\}_n$ denote a sequence of star graphs i.e,  $G_n = K_{1,n}$ and let $H_m = L(K_{1,n})$. Then $\{K_{1,n}\}_n \in S_q$. Moreover $\normalfont \density(H_m) = 1$ and $\lim_{m \to \infty} \normalfont \density(H_m) = 1$.
\end{lemma}
\begin{proof}
   Line graphs of star graphs are complete (Lemma \ref{lemma:linegraphs}-\ref{lemma:linegraphs:star}), i.e. $L(K_{1,n}) = K_{n}$. This gives us the desired result.
\end{proof}

Recall the mass-partition defined from the component sizes (Definition \ref{def:compsize}) $\{C_i(G)/|G|\}_{i=1}^\infty$  for a graph $G$.

\begin{proposition}\label{prop:disjointstars}
    Let $\{G_n\}_n$ be a sequence of disjoint star graphs with $|G_n| \to \infty$ and $H_m = L(G_n)$ denote the line graph of $G_n$. Suppose $U$ is a graphon. Then $H_m \to U$ if an only if $U = U_{\bm{p}}^\mathcal{M}$  for some $\bm{p} \in \mathcal{M}$ and $\{ \left(C_i(G_n)/|G_n| \right)_i \} \to \bm{p}$.
\end{proposition}
\begin{proof}
    Line graphs of disjoint star graphs are disjoint clique graphs. In addition, if $\{ \left(C_i(G_n)/|G_n| \right)_i \} \to \bm{p}$ then $\{ \left(C_i(H_m)/|H_m| \right)_i \} \to \bm{p}$. Thus, the result follows from Theorem \ref{thm:Janson1} (\cite{Janson2016321} Thm 7.5). 
\end{proof}

Proposition \ref{prop:disjointstars} shows that converging line graphs of disjoint stars can be denoted by a mass-partition and the associated graphon. 

\begin{remark}
    Both single stars and multiple disjoint stars $\{G_n\}_n$ give rise to $W = 0$. However their line graphs $\{H_m\}_m$ give rise to different graphons $U$ as apparent from Proposition \ref{prop:disjointstars}. This is an example of differentiating sparse graphs in the line graph space. See Figure \ref{fig:1to4stars}.
\end{remark}

\subsection{Superlinear preferential attachment graphs}\label{sec:prefattachment}
Preferential attachment models \citep{albert2002statistical} consider nodes connecting  to more connected nodes with higher probability. Specifically the probability $\Pi(i)$ that a new node connects to node $i$, which has degree $k_i$ is given by 
\begin{equation}\label{eq:superlinear}
    \Pi(i) = \frac{k_i^\alpha}{\sum_i k_i^\alpha} \, ,  
\end{equation}
where $\alpha$ is a parameter. The three regimes $\alpha < 1$,  $\alpha = 1$ and $\alpha > 1$ are called sublinear, linear and superlinear preferential attachment respectively. Suppose we start with $s_0$ nodes and $s_0$ edges and at each time step $t$ a new node is added to the network with $s$ edges. After $t$ timesteps the network has %$n = t + s_0$ nodes and $m = s_0 + ts$ edges.  
\begin{equation}\label{eq:superlinearnodesedges}
    n = t + s_0 \quad \text{nodes and } \quad  m = s_0 + ts \quad \text{edges.}
\end{equation}

 For growing networks with superlinear preferential attachment \cite{Krapivsky2001, Krapivsky2011} state that the maximum degree $k_{\max}$ satisfies 
$$ k_{\max} \sim n \, . 
$$
\cite{Sethuraman2019} prove a more rigorous version of the above statement. They show that,  
$$  P\left[ \lim_{n \to \infty} \frac{1}{n} k_{\max} = 1 \right] = 1 \, . 
$$
We will use this result to show that superlinear preferential attachment graphs satisfy the square-degree property almost surely. 

\begin{lemma}\label{lemma:superlinear}
    Let $\{G_n\}_n$ denote a sequence of graphs growing by superlinear preferential attachment satisfying equation \eqref{eq:superlinear} with $\alpha > 1$. Then  $ \{G_n\}_n \in S_q $ almost surely. 
  %  $$ P\left[ \{G_n\}_n \in S_q \right] = 1 \, .$$ 
\end{lemma}
\begin{proof}
Using the result from \cite{Sethuraman2019} we know that for every $\epsilon > 0$ there exists $N_0 \in \mathbb{N}$ such that  
$$ P\left[  \left \vert  \frac{1}{n} k_{\max} - 1\right \vert  < \epsilon  \right ] = 1\, , 
$$
for all $n > N_0$. That is, almost surely
$$ 1 - \epsilon  < \frac{1}{n} k_{\max}  < 1 + \epsilon \, , 
$$
for $n > N_0$.
Rearranging the equations for $n$ and $m$ (equation \eqref{eq:superlinearnodesedges}) we get $ns = m + (s-1)s_0$ giving us $ns > m$. Hence, 
$$  \sum (\deg v_i^2) >   k_{\max}^2 > (1 - \epsilon)^2 n^2  > \frac{(1 - \epsilon)^2 m^2}{s^2} \quad \text {almost surely.}
$$
Thus, for $n > N_0$
$$ P \left[ \sum (\deg v_i^2) > \frac{(1 - \epsilon)^2 }{s^2} m^2 \right] = 1
$$
showing that superlinear preferential attachment graphs satisfy the square-degree property (Definition \ref{def:square}) almost surely for large values of $n$. From Theorem \ref{thm:main} they produce dense line graphs. If $\{ t(F, H_m)\}_m$ converges for all graphs $F$, where $H_m = L(G_n)$ then Theorem \ref{thm:BCconvergent} \citep{borgs2008convergent} ensures $\{H_m\}_m$ converges to a graphon $U$. As $\{H_m\}_m$ is dense $U \neq 0$.
\end{proof}

\section{$(U,W)$-mixture graphs}\label{sec:UWMixtures}
% ----------------------------------------------------------------

$(U,W)$-mixture graphs were introduced in \cite{kandOng2025graphonMixtures}, where the central idea is joining two graphs sampled from the graphons $U$ and $W$ where $U$ is a disjoint clique graphon. Recall that a disjoint clique graphon can be represented by a mass-partition $\bm{p}=(p_1, p_2, \ldots )$ (Definition \ref{def:masspartitionJanson}) and a graph sampled from a disjoint clique graphon is a collection of disjoint cliques with the proportion of nodes in the $i^\text{th}$ largest clique given by $p_i$.  As the line graph of a star graph is complete, the root (Definition \ref{def:linegraphroot}) of a complete graph for $n > 3$ is a star graph. Thus, the root graph of a collection of disjoint cliques is a collection of disjoint stars. $(U,W)$-mixture graphs join graphs sampled from $W$ with this collection of stars.  The key is that  the root graph  of the graph sampled from $U$ (or $\bm{p}$) is joined with the graph sampled from $W$ making it possible to generate sparse graph sequences.

% \cheng{I wonder whether we need to tidy up the words ``complete'' vs ``clique'' vs ``single clique''? My feeling is that it is ok to have this messiness. The disjoint clique graphon will be unfamiliar to readers, so a bit more hand holding is probably good here. The easiest is to refer to relevant parts of Section 2. The inverse line graph is not a known concept. We don't define disjoint clique graphon, probably a sentence in Section 2 somewhere? I don't recall where we enforce that $U$ is always disjoint clique (in definition below).} 
% the key is that  the root graph (inverse line graph) of the graph sampled from $U$ is joined with the graph sampled from $W$ making it possible to generate sparse graph sequences. % Given two graphons $U$ and $W$, where $U$ is a disjoint clique graphon and $W$ is any graphon, two graphs are sampled from both graphons. 

\begin{definition} \label{def:WURandomMixtureGraphs}
    Given a graphon $W$, a disjoint clique graphon $U$, and positive, integer valued sequences $\{n_{d_i}\}_i$ and $\{m_{s_i}\}_i$ increasing with $i$, a \textbf{$\bm{(U,W)}$-mixture graph sequence} $G_{n_i} \sim \mathbb{G}\left(U,W, n_{d_i}, m_{s_i}\right)$ is constructed as follows: 
    \begin{enumerate}
        % \item Let $W$ be a graphon and $U$ be a disjoint clique graphon. %and $f:\mathbb{R}^+ \rightarrow \mathbb{R}^+$ is a monotone function. 
        % Let $\{n_{d_i}\}_i$ and $\{m_{s_i}\}_i$ be two positive, increasing, integer valued sequences with both $n_{d_i}$ and $m_{s_i}$ going to infinity. % satisfying $n_{d_i}/m_{s_i} =f(i)$. 
        \item Let $G_{d_i} \sim \mathbb{G}(n_{d_i}, W)$ and $H_{s_i} \sim \mathbb{G}(m_{s_i}, U)$, i.e., $G_{d_i}$ is a $W$-random graph (Definition \ref{def:wrandomgraphs}) and $H_{s_i}$ is a $U$-random graph. Specifically, $H_{s_i}$ is a collection of disjoint cliques with clique proportions satisfying mass-partition $\bm{p}$. Sequences $\{G_{d_i}\}_i$ and $\{H_{s_i}\}_i$ consist of dense graphs. 
     \item As $H_{s_i}$ is a collection of disjoint cliques, there exists a collection of stars $G_{s_i}$ such that $H_{s_i} = L(G_{s_i})$. 
     % As $H_{s_i}$ is a collection of disjoint cliques, the root of $H_{s_i}$ exists and is a disjoint union of stars. %\cheng{This result about stars doesn't seem to be in Section 2.}. 
     % Suppose $G_{s_i} = L^{-1}(H_{s_i})$ has $n_{s_i}$ nodes where $L^{-1}$ denotes the inverse line graph operation. 
     Note that $\{G_{s_i}\}_i$ is a sparse graph sequence. % while $\{G_{d_i}\}_i$ is dense.
     %Thus, we have two graph sequences: $\{G_{d_i}\}_i$, a dense graph sequence and $\{G_{s_i}\}_i$, a sparse graph sequence. 
     \item Join $\{G_{d_i}\}_i$ and $\{G_{s_i}\}_i$ satisfying graph joining rules in Definition \ref{def:joiningrules}. Let $\{G_{n_i}\}_i$ denote the resulting graph sequence. % Let $G_n = \Join(G_d, G_s)$ denote the merged graph of $G_{d}$ and $G_{s}$  (Definition \ref{def:graphmerge}) where $G_n$ has $n$ nodes with $n = \max(n_d, n_s)$. 
    \end{enumerate}
     We say  $G_{n_i} \sim \mathbb{G}\left(U,W, n_{d_i}, m_{s_i}\right)$ is a  $(U,W)$-mixture graph. We call $G_{d_i}$ the \textbf{dense part} and $G_{s_i}$ the \textbf{sparse part} of $G_{n_i}$.  
     %Furthermore, we define the \textbf{mixing function} \cheng{Shall we define mixing functions closer to when we need them?} as $f(n_{d_i}, n_{s_i}) = (f_1(i), f_2(i))$ where $f_1(i) = n_{s_i}/n_{d_i}$ and $f_2(i) = n_{s_i}/n_{d_i}^2$. 
\end{definition}

% \cheng{Singular rule?}
The joining rules below preserve sparsity: if the disjoint union of $\{G_{d_i}\}_i$ and $\{G_{s_i}\}_i$ is sparse, then the joined graph sequence $\{G_{n_i}\}_i$ remains sparse. In other words, the join does not change a sparse disjoint union to a dense graph sequence and vice versa. 
% The joining rules below preserve the sparsity property, that is if the disjoint union of $\{G_{d_i}\}_i$ and  $\{G_{s_i}\}_i$ is sparse, then the joined $\{G_{n_i}\}_i$ is also sparse. That is, the sparse/dense property of the disjoint union is not changed by the join.  %That is, the sparse/dense property of the disjoint union is not changed by the join. 

\begin{definition}\label{def:joiningrules}\textbf{\textbf{(Graph Joining Rules)}}
    Let $\{G_{s_i}\}_i$ and $\{G_{d_i}\}_i$ be two graph sequences. We are interested in joining $G_{s_i}$ and $G_{d_i}$ for every $i$. Let $G_{n_i}$ denote the joined graph. We consider graph joins that satisfy the following conditions. 
    \begin{enumerate}
        \item \textbf{No new nodes}: no new nodes are added as part of the joining process. %This condition ensures that each node in $G_{n_i}$ comes either from $G_{s_i}$ or $G_{d_i}$.
         \item \textbf{No deletion}: No nodes or edges are deleted as part of the joining process.     
     \item \textbf{Random edges}\label{cond:randomedges}: if edges are added, they are added randomly in the sense that nodes within a graph are equally likely to be selected to form edges. %That is, if $G_{s_i}$ and $G_{d_i}$ have $n_{s_i}$ and $n_{d_i}$ nodes respectively, the probability that an edge is added to a given node in $G_{s_i}$, denoted by $p_{edge, G_{s_i}}$ satisfies $p_{edge, G_{s_i}} \leq \frac{1}{n_{s_i}}$ and similarly $p_{edge, G_{d_i}} \leq \frac{1}{n_{d_i}}$.
        % This condition facilitates randomness in the joining process and prevents creating additional structure in the graph. 
     \item \textbf{New edges}: 
        if edges are added, the number of new edges added  $m_{new_i}$ satisfies $m_{new_i} \leq c' m_{d_i}$, for small $c' \in \mathbb{R}$, where $ m_{d_i}$ denotes the number of edges in $G_{d_i}$. %As different values of $c$ can be chosen it gives flexibility to the joining process, this condition stipulates an upper bound on the edges, limiting the influence of the joining process.
    
        % \item \textbf{Rank preserving condition}: with high probability, the joining process preserves the ranks of the top $k$ degrees in $G_{s_i}$ for some $k \in \mathbb{N}$. That is, with high probability, the highest degree node in $G_{s_i}$ becomes the highest degree node in $G_{n_i}$ after the joining process and this holds for the top $k$ node degrees. This condition facilitates the recoverability of $G_{s_i}$ from $G_{n_i}$. 
    \end{enumerate}
\end{definition}

\cite{kandOng2025graphonMixtures}(Lemma 3.4) show that if $\frac{n_{s_i}}{n_{d_i}} \to c < \infty$ then the resulting graph sequence $\{G_{n_i}\}_i$ is dense but if $\frac{n_{s_i}}{n_{d_i}} \to \infty$ then it is sparse. Here we explore the convergence of $(U,W)$-mixture graphs in terms of graphons.  

In Theorem  \ref{thm:UWMixtureConvergence} we consider the disjoint union of $G_{d_i}$ and $G_{s_i}$ and find expressions for the graphons which $G_{n_i}$ and $H_{m_i} = L(G_{n_i})$ converge to. We use the notation developed in \cite{janson2008connectedness} (see Definition \ref{def:oplus} and Figure \ref{fig:graphonAddition}) to express the graphon of the joined graph sequence. 
In Theorem \ref{thm:UWMixtureConvergence2} we go beyond disjoint unions and consider mixtures with a small number of new edges, specifically $m_{new_i} \in o(n_{d_i})$.  This mixture along with disjoint unions form a subset of $(U,W)$-mixture graphs because the number of new edges in Definition \ref{def:joiningrules} satisfy $m_{new_i} \leq c' m_{d_i} \in \Theta(n_{d_i}^2)$. 
% The  mixture graphs considered in this section with $m_{new_i} = 0$ (disjoint unions) and $m_{new_i} \in o(n_{d_i})$, form a subset of all $(U,W)$-mixture graphs. 

\begin{theorem}\label{thm:UWMixtureConvergence}
Let $\{G_{n_i}\}_i$ denote a sequence of $(U,W)$-mixture graphs (Definition \ref{def:WURandomMixtureGraphs}) where the disjoint union is considered, i.e., $G_{n_i} = G_{d_i} \cup G_{s_i}$. Let $H_{m_i} = L(G_{n_i})$ denote the line graph of the mixture graph and recall the notation $\oplus$ in Definition \ref{def:oplus}. 
\begin{enumerate}
    \item Suppose $\frac{n_{s_i}}{n_{d_i}} \to c < \infty$. Then $G_{n_i} \to  \frac{1}{c+1}W \oplus \frac{c}{c+1} 0 $ and  $H_{m_i} \to 0$, where $0$ denotes the zero graphon. Furthermore, $\{G_{n_i}\}_i \in D$. 
    \item Suppose $\frac{n_{s_i}}{n_{d_i}} \to \infty$ and let $\rho = \int_0^1\int_0^1 W \,  dx  dy $.
    \begin{enumerate}
        \item If $\frac{n_{s_i}}{n_{d_i}^2} \to c < \infty$. Then $H_{m_i} \to  \frac{c}{c+\rho} U \oplus \frac{\rho}{c+\rho} 0 $ and $G_{n_i} \to 0$ where $0$ denotes the zero graphon. Furthermore, $\{G_{n_i}\}_i \in S_x$. 
       \item If $\frac{n_{s_i}}{n_{d_i}} \to \infty$ and $\frac{n_{s_i}}{n_{d_i}^2} \to 0$. Then $H_{m_i} \to 0$ and $G_{n_i} \to 0$. Furthermore, $\{G_{n_i}\}_i \in S \backslash S_x$. 
    \item If $\frac{n_{s_i}}{n_{d_i}} \to \infty$ and $\frac{n_{s_i}}{n_{d_i}^2} \to \infty$. Then $H_{m_i} \to U$ and $G_{n_i} \to 0$. Furthermore, $\{G_{n_i}\}_i \in S_x$.  
    \end{enumerate}      
\end{enumerate}
\end{theorem}
\begin{proof}
Let $H_{d_i} = L(G_{d_i})$ denote the line graph of the dense part $G_{d_i}$. 
Recall that $H_{s_i} = L(G_{s_i})$ by definition. 
As $G_{d_i}$ and $H_{s_i}$ are sampled from $W$ and $U$ respectively, $G_{d_i} \to W$ and $H_{s_i} \to U$.

We will use the mapping in Figure \ref{fig:mappingWtoU} which includes results from Lemmas \ref{lemma:densegraphs}, \ref{lemma:converginginSq} and \ref{lemma:convergingSnotSq} to determine where the sequence $\{G_{n_i}\}_i$ lives.

\begin{enumerate}
    \item When $\frac{n_{s_i}}{n_{d_i}} \to c < \infty$, from \cite{kandOng2025graphonMixtures} Lemma 3.4 the sequence $\{G_{n_i}\}_i$ is dense. Thus, $\{G_{n_i}\}_i \in D$ and $H_{m_i} \to 0$. As $\frac{n_{s_i}}{n_{s_i} + n_{d_i}} \to \frac{c}{c+1} $ from Theorem \ref{thm:jansonGraphonAddition}
    \[
    G_{n_i} = G_{d_i} \cup G_{s_i} \to \frac{c}{c+1} 0 \oplus  \frac{1}{c+1} W \, .
    \]
    \item When $\frac{n_{s_i}}{n_{d_i}} \to  \infty$ the sequence $G_{n_i}$ is sparse making  $G_{n_i} \to 0$.  The line graph 
    \[
    H_{m_i} = L(G_{n_i}) = L(G_{d_i} \cup G_{s_i}) = L(G_{d_i}) \cup L(G_{s_i}) = H_{d_i} \cup H_{s_i} \, . 
    \]
    As $G_{d_i}$ is the dense part, $\{H_{d_i}\}_i$ is sparse making $H_{d_i} \to 0$.  Let $H_{d_i}$ have $m_{d_i}$ nodes and $H_{s_i}$ have $m_{s_i}$ nodes. As $H_{d_i}$ and $H_{s_i}$ are line graphs, $m_{d_i}$ and $m_{s_i}$ correspond to the number of edges in $G_{d_i}$ and $G_{s_i}$ respectively. The ratio
    \[
    \frac{m_{d_i}}{m_{d_i} + m_{s_i}} \to \frac{\rho n_{d_i}^2 }{\rho n_{d_i}^2 + n_{s_i}} = \frac{\rho}{\rho + \frac{n_{s_i}}{n_{d_i^2}} }
    \]
    for large $i$. When $\frac{n_{s_i}}{n_{d_i^2}} \to c$ from Theorem \ref{thm:jansonGraphonAddition} we have
    \[
    H_{m_i} \to \frac{\rho}{\rho + c} 0 \oplus \frac{c}{\rho + c} U \, . 
    \]
    
    \begin{enumerate}
        \item This limit remains as it is for $0 < c < \infty$. As $G_{n_i} \to 0$ and $H_{m_i} \to \tilde{U} \neq 0$ we have $\{G_{n_i}\}_i \in S_x$.  
        \item When $c = 0$ we get $H_{m_i} \to 0$ resulting in $\{G_{n_i}\}_i \in S \backslash S_x$.
        \item When $\frac{n_{s_i}}{n_{d_i^2}} \to \infty$ we get $H_{m_i} \to U$ resulting in $\{G_{n_i}\}_i \in S_x$.
    \end{enumerate}
\end{enumerate}
\end{proof}

\begin{theorem}\label{thm:UWMixtureConvergence2}
    Let $\{G_{n_i}\}_i$ denote a sequence of $(U,W)$-mixture graphs (Definition \ref{def:WURandomMixtureGraphs}) where the new edges $m_{new_i} \in o(n_{d_i})$. Let $\tilde{G}_{n_i} = G_{d_i} \cup G_{s_i}$ denote the disjoint union of the same dense and the sparse parts and suppose $\tilde{H}_{m_i} = L(\tilde{G}_{n_i})$. Then $G_{n_i}$ and $\tilde{G}_{n_i}$ converge to the same graphon, and $H_{m_i}$ and $\tilde{H}_{m_i}$ converge to the same graphon. 
\end{theorem}
\begin{proof} As in Theorem \ref{thm:UWMixtureConvergence} we consider two cases: $\frac{n_{s_i}}{n_{d_i}} \to c < \infty$ and $\frac{n_{s_i}}{n_{d_i}} \to \infty$. 
     
\begin{enumerate}
    \item Suppose $\frac{n_{s_i}}{n_{d_i}} \to c < \infty$.  From Theorem \ref{thm:UWMixtureConvergence}(1) we know that the disjoint union satisfies
    \[
    \tilde{G}_{n_i} \to \frac{c}{c+1} 0 \oplus  \frac{1}{c+1} W \, ,
    \]
    and $\tilde{H}_{m_i} \to 0$ where $\tilde{H}_{m_i} = L(\tilde{G}_{n_i})$.  Let $\tilde{W}_{G_i}$ and ${W}_{G_i}$ denote the empirical graphons (Definition \ref{def:empiricalgraphon}) of $\tilde{G}_{n_i}$ and $G_{n_i}$ respectively. Next we show that the cut-metric $\delta_\square(\tilde{W}_{G_i},{W}_{G_i})$ goes to zero. 

    The empirical graphons ${W}_{G_i}$ and $\tilde{W}_{G_i}$ have the same entries in the subsets $[0, n_{d_i}/(n_{d_i} + n_{s_i})]^2$ and $[n_{d_i}/(n_{d_i} + n_{s_i}), 1]^2$ assuming we label the nodes starting with the dense part first.  The differences are in the subset $[0, n_{d_i}/(n_{d_i} + n_{s_i})] \times [n_{d_i}/(n_{d_i} + n_{s_i}), 1]$ and the corresponding subset below the diagonal. The area of this subset is $2 n_{d_i} n_{s_i}/(n_{d_i} + n_{s_i})^2$. This subset has the value 0 in   $\tilde{W}_{G_i}$ as it is the disjoint union of $G_{d_i}$ and $G_{s_i}$.  In ${W}_{G_i}$  $m_{new_i}$ edges are present in this subset bounding the cut metric (Definition \ref{def:cut2})
    \[
    \delta_\square(\tilde{W}_{G_i},{W}_{G_i}) = \inf_\phi \lVert  \tilde{W}_{G_i} - W_{G_i}^\phi\rVert_{\square} \leq \lVert  \tilde{W}_{G_i} - W_{G_i}\rVert_{\square} =  \frac{2 m_{new_i}}{n_{d_i} n_{s_i}}\, .  % \frac{2 m_{new_i}/(n_{d_i} + n_{s_i})^2 }{n_{d_i} n_{s_i}/(n_{d_i} + n_{s_i})^2}  =
    \]
     When $m_{new_i} \in o(n_{d_i})$ we get $\delta_\square(\tilde{W}_{G_i},{W}_{G_i}) \to 0$ as $i$ goes to infinity making $G_{n_i}$ converge to the same graphon as $\tilde{G}_{n_i}$. As $\{G_{n_i} \}_i \in D$ and $\{H_{m_i} \}_i \in S$ we have $H_{m_i} \to 0$ similar to $\tilde{H}_{m_i}$.

      \item Next consider the case $\frac{n_{s_i}}{n_{d_i}} \to \infty$. In this case we know the mixture sequence $\{G_{n_i}\}_i$ is sparse from \cite{kandOng2025graphonMixtures} Lemma 3.4. So both $G_{n_i}$ and $\tilde{G}_{n_i}$ converge to the zero graphon.   As before let $\tilde{H}_{m_i} = L(\tilde{G}_{n_i})$ denote the line graph of the disjoint union. We introduce a new graph 
      \[
      \breve{G}_{n_i} =\tilde{G}_{n_i} \cup_{i=1}^{m_{new_i}} K_{1,1}  \, , 
      \] 
      where $\breve{G}_{n_i}$ denotes the graph with $m_{new_i}$ disjoint edges ($K_{1,1}$) added to the disjoint union $\tilde{G}_{n_i}$. Let $\breve{H}_{m_i} = L(\breve{G}_{n_i})$ denote the line graph of $\breve{G}_{n_i}$. As the line graph operation keeps connected components connected we have
      \[
      \breve{H}_{m_i}  = \tilde{H}_{m_i} \cup_{i = 1}^{m_{new_i}} K_{1} \, , 
      \]
    where $K_1$ denotes the complete graph with one vertex, i.e., $K_1 = \vertex$ \, ,  an isolated vertex. That is, $\breve{H}_{m_i} $ has $m_{new_i}$ isolated vertices added to $\tilde{H}_{m_i}$. First we show that $\tilde{H}_{m_i}$ and $\breve{H}_{m_i}$ converge to the same graphon using homomorphism densities. Then we show that $\breve{H}_{m_i}$ and $H_{m_i}$ converge to the same graphon using the cut metric. 
    
    Consider the homomorphism densities $t(F, \tilde{H}_{{m_i}})$  and $t(F, \breve{H}_{{m_i}})$ (Definition \ref{def:graphhomo}) for a simple graph $F$. For $F = K_1 = \vertex $ \,  both  $t(\vertex, \tilde{H}_{m_i}) = t(\vertex, \breve{H}_{m_i}) = 1$. For all other simple  $F$
    \begin{align}
        t(F, \tilde{H}_{m_i}) & = \frac{\Hom(F, \tilde{H}_{m_i})}{|V(\tilde{H}_{m_i} )|^{|V(F)|}} \, , \,  \text{and}\\
        t(F, \breve{H}_{m_i}) & = \frac{\Hom(F, \breve{H}_{m_i})}{|V(\breve{H}_{m_i} )|^{|V(F)|}}  \, , \\
        & = \frac{\Hom(F, \tilde{H}_{m_i})}{|V(\breve{H}_{m_i} )|^{|V(F)|}}  \, ,  \\
        & = \frac{\Hom(F, \tilde{H}_{m_i})}{(|V(\tilde{H}_{m_i} )| + m_{new_i})^{|V(F)|}}
    \end{align}
  where we have used the fact that $\Hom(F, \tilde{H}_{m_i}) = \Hom(F, \breve{H}_{m_i})$ for all $F \neq \vertex$.  
  As $|V(\tilde{H}_{m_i} )| \in \Theta(n_{d_i}^2)$ and $m_{new_i} \in o(n_{d_i})$ we have 
  \[
  \frac{m_{new_i}}{|V(\tilde{H}_{m_i} )|} \to 0
  \]
  making $t(F, \tilde{H}_{m_i})$ and $t(F, \breve{H}_{m_i})$ both converge to the same limit for all $F \neq \vertex$. For $F = \vertex$ the homomorphism densities equal 1. Therefore both graph sequences  $\{\tilde{H}_{m_i}\}_i$ and  $\{\breve{H}_{m_i}\}_i$ have the same limit, which we know from Theorem \ref{thm:UWMixtureConvergence}(2) to be $\frac{\rho}{\rho + c} 0 \oplus \frac{c}{\rho + c} U $.

  Let $U_{H_i}$ and $\breve{U}_{H_i}$ denote the empirical graphons of $H_{s_i}$ and $\breve{H}_{s_i}$. Next, we show that the cut metric between $U_{H_i}$ and $\breve{U}_{H_i}$ goes to zero. The cut metric  (Definition \ref{def:cut2}) between two empirical graphons is given by
  \[
    \delta_{\square}(U_{H_i}, \breve{U}_{H_i}) = \inf_\phi \lVert U_{H_i} - \breve{U}_{H_i}^{\phi}  \rVert_\square \leq \lVert U_{H_i} - \breve{U}_{H_i} \rVert_\square \, . 
  \]
  Recall $H_{m_i}$ has $m_{d_i} + m_{s_i} + m_{new_i}$ nodes, where $m_{d_i}$ is the number of edges in the dense part $G_{d_i}$, and $m_{s_i}$ is the number of edges in the sparse spart $G_{s_i}$ and $m_{new_i}$ is the number of newly added edges as part of the joining process.  Assuming we label the nodes in $H_{m_i}$ starting with the edges in the dense part, followed by the edges in the sparse part and finally the new added edges $m_{new_i}$, the empirical graphons $U_{H_i}$ and $\breve{U}_{H_i}$ are only different in the interval $[0, (m_{d_i} + m_{s_i})/(m_{d_i} + m_{s_i} + m_{new_i})] \times [(m_{d_i} + m_{s_i})/(m_{d_i} + m_{s_i} + m_{new_i}), 1]$. This gives us
  \begin{align}
      \lVert U_{H_i} - \breve{U}_{H_i} \rVert_\square 
      &  \leq \frac{2 m_{new_i} (m_{d_i} + m_{s_i})}{(m_{d_i} + m_{s_i} + m_{new_i})^2}  = \frac{2  m_{new_i} }{\left(1  + \frac{m_{new_i}}{m_{d_i} + m_{s_i}  } \right)( m_{d_i} + m_{s_i} + m_{new_i})}\, . 
 \end{align}
As $m_{d_i} \in \Theta(n_{d_i}^2)$ and $m_{new_i} \in o(n_{d_i})$ we have 
\[
\delta_{\square}(U_{H_i}, \breve{U}_{H_i}) \leq  \lVert U_{H_i} - \breve{U}_{H_i} \rVert_\square  \to 0 \, . 
\]
Thus, $U_{H_i}$ and $ \breve{U}_{H_i} $ converge to the same graphon $\frac{\rho}{\rho + c} 0 \oplus \frac{c}{\rho + c} U $. 
\end{enumerate}

\end{proof}

% ---------------------------
We have explored a subset of $(U,W)$-mixture graphs focusing on disjoint unions and new edges satisfying $m_{new_i} \in o(n_{d_i})$ and have shown that depending on the limits of the ratios $\frac{n_{s_i}}{n_{d_i}}$ and $\frac{n_{s_i}}{n_{d_i}^2}$ the mixture graphs can live in $D$, $S\backslash S_x$ or in $S_x$. Figure \ref{fig:withUWMixtures} gives an updated Euler diagram including this subset of $(U,W)$-mixture graphs.

\begin{figure}[!ht]
    \centering
    \begin{tikzpicture}
        \draw (0,0) rectangle (5cm, 2.25cm);
        \draw (3,1.125) ellipse (1.5cm and 0.75cm);
        \draw (3,1.125) ellipse (0.75cm and 0.25cm);
        \draw[fill=gray!30] (2.0,1.125) ellipse (1.2cm and 0.25cm);
        % \draw (1.5,1.75) circle [radius=1];
        % \draw (4.5,1.75) circle [radius=1];
        \node[] at (3,1.65) { $S$};
        \node[] at (3.4,1.125) {$S_x$};
        \node[] at (2.5,2.5) {Converging $\{G_n\}_n$ and $\{H_m\}_m$};
        \node[] at (1,0.5) { $D$};
         \node[] at (2,1.125) {$(U,W)$-mix.};
    \end{tikzpicture}
    \caption{The Euler diagram in Figure \ref{fig:vennconverging} updated with a subset of $(U,W)$-mixture graphs.  }
    \label{fig:withUWMixtures}
\end{figure}
\section{Further questions}\label{sec:conclusions}

A common thread between $(U,W)$-mixture graphs and graphexes~\citep{veitch2015class,borgs2021limits} is that both constructions add stars. Graphexes have a star function $S$ that adds edges to each node depending on its `starness' given by $S(\theta)$ where $\theta$ is the latent position of the node.  In contrast, $(U,W)$-mixture graphs have a mass-partition $\bm{p}=(p_1, p_2, \ldots)$ that add stars given by the proportions of $\bm{p}$. Even though the mass-partition can have infinite elements, $\bm{p}$ is discrete and is different from $S$, which is  defined on a  continuous domain. %Furthermore, graphexes were originally constructed using Kallenberg exchangeability. 
%Even though the mass-partition can have infinite elements, it is different from $S(\theta)$, which stipulates the starness of every node. 
Exploring whether a connection exists between these two constructions offers a direction of inquiry.  

We have investigated line graph limits and shown that a subset of sparse graphs have dense line graphs, endowing a structure on the space of converging graph sequences. In addition to the line graph, there can be other transformations that map subsets of sparse graphs to dense graphs. %What other transformations map sparse graphs to dense graphs? 
For example, are there transformations that map graphs in $S \backslash S_x$ such as cycles and paths to dense graphs? This is a question of interest. More generally, what properties are needed in a transformation so that the transformed graph reveals novel structures that converge to a limit?

\bibliographystyle{agsm}
\bibliography{references}

\end{document}